\documentclass[twoside]{article}
\usepackage[accepted]{aistats2019}
\usepackage[round]{natbib}

\usepackage{glossaries}

\usepackage{amsmath}
\usepackage{amssymb}
\usepackage{amsthm}
\usepackage{mathbbol}
\usepackage{mathtools}
\usepackage{mathrsfs}
\usepackage{vector}
\usepackage{cleveref}
\usepackage{bm}
\usepackage[dvipsnames]{xcolor}
\newtheorem{theorem}{Theorem}[section]

\newtheorem{lemma}[theorem]{Lemma}

\usepackage{multirow}
\usepackage{algorithm}
\usepackage{algorithmic}
\newcommand{\argmax}{\operatornamewithlimits{argmax}}
\numberwithin{equation}{section}
\numberwithin{table}{section}
\numberwithin{algorithm}{section}

\newacronym{KME}{KME}{kernel mean embedding}
\newacronym{CME}{CME}{conditional mean embedding}
\newacronym{CKM}{CKM}{conditional kernel mean}
\newacronym{RKHS}{RKHS}{reproducing kernel Hilbert space}
\newacronym{SVM}{SVM}{support vector machine}
\newacronym{SVC}{SVC}{support vector classifier}
\newacronym{KMCF}{KMCF}{kernel Monte Carlo filter}
\newacronym{PDF}{PDF}{probability density function}
\newacronym{CDF}{CDF}{cumulative distribution function}
\newacronym{GP}{GP}{Gaussian process}
\newacronym{GPs}{GPs}{Gaussian processes}
\newacronym{GPC}{GPC}{Gaussian process classifier}
\newacronym{GPR}{GPR}{GP regressor}
\newacronym{BO}{BO}{Bayesian optimization}
\newacronym{KRR}{KRR}{kernel ridged regressor}
\newacronym{RLSC}{RLSC}{regularized least squares classification}
\newacronym{ERM}{ERM}{empirical risk minimization}
\newacronym{KBR}{KBR}{kernel Bayes' rule}
\newacronym{MMD}{MMD}{maximum mean discrepancy}
\newacronym{CMMD}{CMMD}{conditional maximum mean discrepancy}
\newacronym{RCB}{RCB}{Rademacher complexity bound}
\newacronym{OVA}{OVA}{one versus all}
\newacronym{CNN}{CNN}{convolutional neural network}
\newacronym{ABC}{ABC}{approximate Bayesian computation}
\newacronym{MC}{MC}{Monte Carlo}
\newacronym{MCMC}{MCMC}{Markov chain Monte Carlo}
\newacronym{SMC}{SMC}{sequential Monte Carlo}
\newacronym{SMC-ABC}{SMC-ABC}{sequential Monte Carlo ABC}
\newacronym{MH}{MH}{Metropolis Hastings}
\newacronym{BLR}{BLR}{Bayesian linear regression}
\newacronym{VI}{VI}{variational inference}
\newacronym{KL}{KL}{Kullback-Leibler}
\newacronym{HS}{HS}{Hilbert Schmidt}
\newacronym{MLE}{MLE}{maximum likelihood estimator}
\newacronym{MAP}{MAP}{maximum a posteriori}
\newacronym{QMC}{QMC}{quasi Monte Carlo}
\newacronym{SVI}{SVI}{stochastic variational inference}
\newacronym{SVGD}{SVGD}{Stein variational gradient descent}
\newacronym{VBIL}{VBIL}{variational Bayes with intractable likelihood}
\newacronym{VBSL}{VBSL}{variational Bayes with synthetic likelihood}
\newacronym{REJ-ABC}{REJ-ABC}{rejection ABC}
\newacronym{KR-ABC}{KR-ABC}{kernel recursive ABC}
\newacronym{K-ABC}{K-ABC}{kernel ABC}
\newacronym{DR-ABC}{DR-ABC}{distribution regression ABC}
\newacronym{K2-ABC}{K2-ABC}{double kernel ABC}
\newacronym{ASL-ABC}{ASL-ABC}{adaptive SL-ABC}
\newacronym{SL-ABC}{SL-ABC}{synthetic likelihood ABC}
\newacronym{BSL-ABC}{BSL-ABC}{Bayesian synthetic likelihood ABC}
\newacronym{AV-ABC}{AV-ABC}{automatic variational ABC}
\newacronym{GPS-ABC}{GPS-ABC}{GP surrogate ABC}
\newacronym{GPA-ABC}{GPA-ABC}{GP-accelerated ABC}
\newacronym{HIM}{HIM}{hierarchical implicit models}
\newacronym{LFVI}{LFVI}{likelihood-free variational inference}
\newacronym{KML}{KML}{kernel means likelihood}
\newacronym{MKML}{MKML}{marginal kernel means likelihood}
\newacronym{KMP}{KMP}{kernel means posterior}
\newacronym{KMPE}{KMPE}{kernel means posterior embedding}
\newacronym{ST-KML}{ST-KML}{spatio-temporal kernel means likelihood}
\newacronym{KMS-ABC}{KML-ABC}{kernel means surrogate ABC}
\newacronym{CKM-ABC}{CKM-ABC}{conditional kernel means ABC}
\newacronym{KMS}{KMS}{kernel means surrogate}
\newacronym{ARD}{ARD}{automatic relevance determination}
\newacronym{LOTUS}{LOTUS}{law of the unconscious statistician}
\newacronym{LF-MCMC}{LF-MCMC}{likelihood-free Monte Carlo Markov chain}
\newacronym{MDN}{MDN}{mixture density network}
\newacronym{MSE}{MSE}{mean squared error}
\newacronym{NMSE}{NMSE}{normalized MSE}
\newacronym{KDE}{KDE}{kernel density estimation}
\newacronym{LFI}{LFI}{likelihood-free inference}
\newacronym{KELFI}{KELFI}{kernel embedding likelihood-free inference}

\begin{document}

%
\runningtitle{Bayesian Learning of Conditional Kernel Mean Embeddings for Automatic Likelihood-Free Inference}

%
\runningauthor{Kelvin Hsu and Fabio Ramos}

\twocolumn[

\aistatstitle{Bayesian Learning of Conditional Kernel Mean Embeddings \\ for Automatic Likelihood-Free Inference}

\aistatsauthor{Kelvin Hsu$^{1}$$^{2}$ \And Fabio Ramos$^{1}$$^{3}$}

\aistatsaddress{$^{1}$School of Computer Science, The University of Sydney, Australia. $^{2}$CSIRO, Australia. $^{3}$NVIDIA, USA.}]

\begin{abstract}
	
	In likelihood-free settings where likelihood evaluations are intractable, approximate Bayesian computation (ABC) addresses the formidable inference task to discover plausible parameters of simulation programs that explain the observations. However, they demand large quantities of simulation calls. Critically, hyperparameters that determine measures of simulation discrepancy crucially balance inference accuracy and sample efficiency, yet are difficult to tune. In this paper, we present kernel embedding likelihood-free inference (KELFI), a holistic framework that automatically learns model hyperparameters to improve inference accuracy given limited simulation budget. By leveraging likelihood smoothness with conditional mean embeddings, we nonparametrically approximate likelihoods and posteriors as surrogate densities and sample from closed-form posterior mean embeddings, whose hyperparameters are learned under its approximate marginal likelihood. Our modular framework demonstrates improved accuracy and efficiency on challenging inference problems in ecology.
\end{abstract}

\section{Introduction}

	Scientific understanding of complex phenomena are deeply reliant on the study of probabilistic generative models and their match with real world data. Often, latent and convoluted interactions result in intractable likelihood evaluations, making the setting \textit{likelihood-free}. Instead, generative models are expressed as a stochastic forward model simulator. Inference on latent variables in this setting is particularly challenging.

	\Gls{ABC} methods are the state-of-the-art in simulation-based Bayesian inference with intractable likelihoods \citep{marin2012approximate}. They infer posterior distributions of simulator parameters that aim to explain observed data. The posterior is of interest in its own right for understanding the complex phenomena, and also useful in forming predictions of future observations. They are popular due to their simplicity and applicability, and have been used extensively in the biological sciences \citep{beaumont2010approximate, toni2009approximate}. Nevertheless, complex models are often prohibitively expensive to simulate. Evolutionary processes of ecological systems, vibrational modes of a mechanical structure, and fluid flow across surfaces are all examples that result in formidable inference problems with demanding forward simulations. It is thus imperative for inference algorithms to perform under the constraint of limited simulation calls, posing an exceptionally challenging task.
	
	Often, \gls{ABC} methods rely on discrepancy measures between simulations and observations that are parametrized by hyperparameters such as $\epsilon$. The resulting posterior approximation is highly sensitive to the choice of hyperparameters, yet appropriate hyperparameter tuning strategies remain to be established.
	
	To address these issues, we present \gls{KELFI}, a holistic framework consisting of (1) a consistent surrogate likelihood \textit{model} that modularizes queries from simulation calls, (2) a Bayesian \textit{learning} objective for hyperparameters that improves inference accuracy, and (3) a posterior surrogate density and a super-sampling \textit{inference} algorithm using its closed-form posterior mean embedding.
	
	\Gls{KELFI} is based on approximating likelihoods with simulation samples using \glspl{CME}. \Glspl{CME} encode conditional expectations empirically by leveraging smoothness within a \gls{RKHS} with only a small number of examples. This modularizes inference away from simulation calls. Consequently, scientists can proceed with posterior analysis after any number of simulations. Furthermore, \gls{KELFI} infers both approximate posterior densities and samples. Critically, our learning algorithm tunes hyperparameters directly for the inference problem, including adapting $\epsilon$ to the number of simulations used. This removes the need for practitioners to arduously select hyperparameters. Finally, it can be extended to automatically learn the relevance and usefulness of each summary statistic. 
		 

\section{Likelihood-Free Inference}
	
	In the likelihood-free setting, we begin with a stochastic forward model simulator which synthesizes simulations $\bvec{x}$ given a parameter setting $\bm{\theta}$. Let $\bm{\theta} \in \vartheta$ denote a realization of the latent variable or parameter $\bm{\Theta}$, where we use upper cases to denote random variables. Let $\bvec{x} \in \mathcal{X}$ and $\bvec{y} \in \mathcal{Y}$ where $\mathcal{X}, \mathcal{Y} \subseteq \mathcal{D}$ denote realizations of simulation output $\bvec{X}$ and observations $\bvec{Y}$ respectively. We represent the simulator as $p(\bvec{x} | \bm{\theta})$ from which we can only simulate or sample, but not query its density, making likelihood evaluations intractable, thus \textit{likelihood-free}. To begin inference we posit a prior density $p(\bm{\theta})$ that encodes prior knowledge about plausible parameter settings to guide the inference. The goal is to infer a posterior distribution on the parameters $\bm{\theta}$ that could generate simulations $\bvec{x}$ similar to our observations $\bvec{y}$ by some comparison measure. This measure could be done by a standard $\epsilon$-kernel or \gls{ABC} kernel $p_{\epsilon}(\bvec{y} | \bvec{x}) = \kappa_{\epsilon}(\bvec{y}, \bvec{x})$, such as a Gaussian density $\mathcal{N}(\bvec{y} | \bvec{x}, \epsilon^{2} I)$ \citep{price2017bayesian, moreno2016automatic}.
		
	Based on this formulation, the true full likelihood of our model can be written as follows,
	\begin{equation}
		p_{\epsilon}(\bvec{y} | \bm{\theta}) = \int_{\mathcal{X}} p_{\epsilon}(\bvec{y} | \bvec{x}) p(\bvec{x} | \bm{\theta}) d\bvec{x} = \mathbb{E}[\kappa_{\epsilon}(\bvec{y} , \bvec{X}) | \bm{\Theta} = \bm{\theta}].
		\label{eq:likelihood}
	\end{equation}
	The corresponding posterior of interest is $p_{\epsilon}(\bm{\theta} | \bvec{y}) = p_{\epsilon}(\bvec{y} | \bm{\theta}) p(\bm{\theta}) / p_{\epsilon}(\bvec{y})$ where $p_{\epsilon}(\bvec{y}) = \int_{\vartheta} p_{\epsilon}(\bvec{y} | \bm{\theta}) p(\bm{\theta}) d\bm{\theta} $. Due to the presence of a non-zero $\epsilon$, even a perfect approximation to the soft posterior $p_{\epsilon}(\bm{\theta} | \bvec{y})$ will not be the exact posterior $p_{\epsilon = 0}(\bm{\theta} | \bvec{y})$ unless $\epsilon$ is annealed to zero. This is the necessary trade-off we make with limited simulations, where a non-zero $\epsilon$ is essential for tractable inference because no simulations will match the observations exactly in practice. If $\bvec{y}$ is only available as a summary statistic, then this soft posterior $p_{\epsilon}(\bm{\theta} | \bvec{y})$ that we are targeting is only an approximation to the posterior given the full data even with $\epsilon = 0$. 
	
	So far, the notation $\bvec{x}$ and $\bvec{y}$ denote either the full dataset or their summary statistics. This is because the summary operation can be appended to the simulator program to output summary statistics directly. In either case, we let the target posterior be $p_{\epsilon}(\bm{\theta} | \bvec{y})$, so the inference problem remains structurally identical. For simplicity however, from here on $\bvec{x}$ and $\bvec{y}$ will denote summary statistics unless stated otherwise.
		
	Since $p(\bvec{x} | \bm{\theta})$ is intractable, so is the likelihood \eqref{eq:likelihood}. Instead, approximations are required. \gls{MCMC} \gls{ABC} use empirical means, $p_{\epsilon}(\bvec{y} | \bm{\theta}) \approx \frac{1}{S} \sum_{s = 1}^{S} p_{\epsilon}(\bvec{y} | \bvec{x}^{(s)}) =  \frac{1}{S} \sum_{s = 1}^{S} \kappa_{\epsilon}(\bvec{y} , \bvec{x}^{(s)})$ \citep{andrieu2009pseudo}. \Gls{SL-ABC} and \gls{ASL-ABC} alternatively use Gaussian approximations $p_{\epsilon}(\bvec{y} | \bm{\theta}) \approx \mathcal{N}(\bvec{y} | \bm{\mu}_{\bm{\theta}}, \bm{\Sigma}_{\bm{\theta}} + \epsilon^{2} I)$ and estimate the mean and covariance from simulations \citep{wood2010statistical}. These approaches require generating $S$ new simulations $\{\bvec{x}^{(s)}\}_{s = 1}^{S}$ corresponding to $\bm{\theta}$ every time the likelihood is queried.
	
	Not only are synthetic likelihoods parametric Gaussian approximations, they also approximate separately at each $\bm{\theta}$. Instead, surrogate likelihood approaches like \gls{KELFI} use consistent nonparametric approximations so that (1) only one new simulation is required at each new parameter $\bm{\theta}$ and (2) likelihood queries do not need to be at parameters where simulations are available.
	
	
\section{Kernel Embedding \newline Likelihood-Free Inference}
	
	We present \gls{KELFI} in three stages. In the model stage, we build a surrogate likelihood model by leveraging smoothness properties of \glspl{CME}. In the learning stage, we derive a differentiable marginal surrogate likelihood to drive hyperparameters learning. In the inference stage, we propose an algorithm to sample from the resulting mean embedding of the surrogate posterior.

	When the prior is an anisotropic Gaussian $p(\bm{\theta}) = \prod_{d = 1}^{D} \mathcal{N}(\theta_{d} | \mu_{d}, \sigma_{d}^{2})$, closed form solutions for \gls{KELFI} exists. We will present this setting since, for many common continuous priors, the \gls{LFI} problem can be transformed into an equivalent problem that involves a Gaussian prior. See \cref{sec:gaussian_prior} for more detail. When this is not possible or preferred, \gls{KELFI} can be approximated arbitrarily well by using arbitrarily many prior samples.

	\subsection{Conditional Mean Embeddings}

		We begin with an overview of \glspl{CME} in the context of \gls{KELFI}. \Glspl{KME} are an arsenal of techniques used to represent distributions in a \gls{RKHS} \citep{muandet2017kernel}. The key object is the mean embedding of a distribution $X \sim \mathbb{P}$ under a positive definite kernel $k$ via $\mu_{X} := \int_{\mathcal{X}} k(x, \cdot) d\mathbb{P}(x) = \int_{\mathcal{X}} k(x, \cdot) p(x) dx \in \mathcal{H}_{k}$, where the last equality assumes a density $p$ for $\mathbb{P}$ exists and $\mathcal{H}_{k}$ denotes the \gls{RKHS} of $k$. They encode distributions in the sense that function expectations can be written as $\mathbb{E}[f(X)] = \langle \mu_{X}, f \rangle_{\mathcal{H}_{k}}$ if $f \in \mathcal{H}_{k}$. When $\mu_{X}$ can only be estimated empirically in some form denoted as $\hat{\mu}_{X}$, the expectation can be approximated by $\mathbb{E}[f(X)] \approx \langle \hat{\mu}_{X}, f \rangle_{\mathcal{H}_{k}}$.
		
		\glspl{CME} are \glspl{KME} that encode conditional distributions. We specifically focus on their empirical estimates as we assume we only have the resource to obtain $m$ sets of simulation data due to budget constraints. This results in joint samples $\{\bm{\theta}_{j}, \bvec{x}_{j}\}_{j = 1}^{m}$ from $p(\bvec{x} | \bm{\theta}) \pi(\bm{\theta})$ by sampling from a proposal prior $\pi$ for $\bm{\theta}_{j} \sim \pi(\bm{\theta})$ and simulating $\bvec{x}_{j} \sim p(\bvec{x} | \bm{\theta}_{j})$ at each $\bm{\theta}_{j}$. Note these samples are not necessarily from the original joint distribution $p(\bvec{x} | \bm{\theta}) p(\bm{\theta})$ if $\pi \neq p_{\bm{\Theta}}$.
		
		We define positive definite and characteristic kernels \citep{sriperumbudur2010relation} $k : \mathcal{D} \times \mathcal{D} \to \mathbb{R}$ and $\ell : \vartheta \times \vartheta \to \mathbb{R}$. When relevant, we denote the hyperparameters of $k$ and $\ell$ with $\alpha$ and $\beta$, and refer to them as $k_{\alpha} = k(\cdot, \cdot; \alpha)$ and $\ell_{\beta} = \ell(\cdot, \cdot; \beta)$. An useful example of such a kernel is an anisotropic Gaussian kernel $\ell(\bm{\theta}, \bm{\theta}'; \bm{\beta}) = \exp{\big(-\frac{1}{2} \sum_{d = 1}^{D} (\theta_{d} - \theta_{d}')^{2} / \beta_{d}^{2}\big)}$ whose hyperparameters are length scales $\bm{\beta} = \{\beta_{d}\}_{d = 1}^{D}$ for each dimension $d \in [D] := \{1, \dots, D\}$, and similarly for $k$. 
		
		For any function $f \in \mathcal{H}_{k}$, we construct an approximation to $\mathbb{E}[f(\bvec{X}) | \bm{\Theta} = \bm{\theta}]$ by the inner product $\langle f, \hat{\mu}_{\bvec{X} | \bm{\Theta} = \bm{\theta}} \rangle_{\mathcal{H}_{k}}$ with an empirical \gls{CME} $\hat{\mu}_{\bvec{X} | \bm{\Theta} = \bm{\theta}}$. Importantly, $\hat{\mu}_{\bvec{X} | \bm{\Theta} = \bm{\theta}}$ is estimated from the \textit{joint} samples $\{\bm{\theta}_{j}, \bvec{x}_{j}\}_{j = 1}^{m}$, even though it is encoding the corresponding conditional distribution $p(\bvec{x} | \bm{\theta})$. This approximation admits the following form \citep{song2009hilbert}, 
		\begin{equation}
			\mathbb{E}[f(\bvec{X}) | \bm{\Theta} = \bm{\theta}] \approx \bvec{f}^{T} (L + m \lambda I)^{-1} \bm{\ell}(\bm{\theta}),
		\label{eq:conditional_kernel_means_function_expectation}
		\end{equation}
		where $\bvec{f} := \{f(\bvec{x}_{j})\}_{j = 1}^{m}$, $L := \{\ell(\bm{\theta}_{i}, \bm{\theta}_{j})\}_{i, j = 1}^{m}$, $\bm{\ell}(\bm{\theta}) := \{\ell(\bm{\theta}_{j}, \bm{\theta})\}_{j = 1}^{m}$, and $\lambda \geq 0$ is a regularization parameter. This approximation is known to converge at $O_{p}(m^{-\frac{1}{4}})$ if $\lambda$ is chosen to decay at $O_{p}(m^{-\frac{1}{2}})$ or better under appropriate assumptions on $p(\bvec{x} | \bm{\theta})$ \citep{song2013kernel}.

	\subsection{Model: Kernel Means Likelihood}

		We begin by presenting our surrogate likelihood model. Since the likelihood \eqref{eq:likelihood} is an expectation under $p(\bvec{x} | \bm{\theta})$, we propose to approximate it via an inner product with the \gls{CME} of $p(\bvec{x} | \bm{\theta})$. Specifically, if we choose $k$ such that $\kappa_{\epsilon}(\bvec{y} , \cdot) \in \mathcal{H}_{k}$, then $p_{\epsilon}(\bvec{y} | \bm{\theta})$ can be approximated by $q(\bvec{y} | \bm{\theta}) := \langle \kappa_{\epsilon}(\bvec{y} , \cdot), \hat{\mu}_{\bvec{X} | \bm{\Theta} = \bm{\theta}} \rangle_{\mathcal{H}_{k}}$. We refer to $q(\bvec{y} | \bm{\theta})$ as the \gls{KML}. While the \gls{KML} provides an asymptotically correct likelihood surrogate, for finitely many simulations it is not necessarily positive nor normalized. By using $f = \kappa_{\epsilon}(\bvec{y} , \cdot)$ in \eqref{eq:conditional_kernel_means_function_expectation} where $\bm{\kappa}_{\epsilon}(\bvec{y}) := \{\kappa_{\epsilon}(\bvec{y}, \bvec{x}_{j})\}_{j = 1}^{m}$ and $\bvec{v}(\bvec{y}) := (L + m \lambda I)^{-1} \bm{\kappa}_{\epsilon}(\bvec{y})$, the \gls{KML} becomes
		\begin{equation}
			q(\bvec{y} | \bm{\theta}) = \sum_{j = 1}^{m} v_{j}(\bvec{y}) \ell(\bm{\theta}_{j}, \bm{\theta}).
		\label{eq:kernel_means_likelihood}
		\end{equation}
		The \gls{KML} converges at the same rate as the \gls{CME}. See \cref{thm:kernel_means_likelihood} for proof. It is worthwhile to note that the assumption  $\ell(\bm{\theta}, \cdot) \in \mathrm{image}(C_{\bm{\Theta} \bm{\Theta}})$ is common for \glspl{CME}, and is not as restrictive as it may first appear, as it can be relaxed through introducing the regularization hyperparameter $\lambda$ \citep{song2013kernel}.
		\begin{theorem}
			\label{thm:kernel_means_likelihood_copy}
			Assume $\ell(\bm{\theta}, \cdot) \in \mathrm{image}(C_{\bm{\Theta} \bm{\Theta}})$. The \acrfull{KML} $q(\bvec{y} | \bm{\theta})$ converges to the likelihood $p_{\epsilon}(\bvec{y} | \bm{\theta})$ uniformly at rate $O_{p}((m \lambda)^{-\frac{1}{2}} + \lambda^{\frac{1}{2}})$ as a function of $\bm{\theta} \in \vartheta$ and $\bvec{y} \in \mathcal{Y}$.
		\end{theorem}
		To satisfy $\kappa_{\epsilon}(\bvec{y} , \cdot) \in \mathcal{H}_{k}$, we choose the standard Gaussian $\epsilon$-kernel $\kappa_{\epsilon}(\bvec{y} , \bvec{x}) = \mathcal{N}(\bvec{y} | \bvec{x}, \epsilon^{2} I)$ and let $k_{\alpha} = k_{\epsilon}$ be a Gaussian kernel with length scale $\alpha = \epsilon$. Since $\kappa_{\epsilon}(\bvec{y} , \bvec{x})$ and $k_{\epsilon}(\bvec{y} , \bvec{x})$ are scalar multiples of each other, we have that $\kappa_{\epsilon}(\bvec{y} , \cdot) \in \mathcal{H}_{k}$. In fact, any positive definite kernel $\kappa_{\epsilon}$ can be used, since we can simply choose $k_{\alpha}$ to be its scalar multiple to form the \gls{RKHS}. 
		
		When the raw data is \textit{iid} and no sufficient summary statistics are available, we can employ a kernel on the empirical distributions of the two datasets via $\kappa_{\epsilon, \alpha}(\bvec{y} , \bvec{x}) \propto  k_{\epsilon, \alpha}(\bvec{y} , \bvec{x}) = \exp{\big(-\frac{1}{2 \epsilon^{2}} \| \hat{\mu}_{\bvec{Y}} - \hat{\mu}_{\bvec{X}}\|_{\mathcal{H}_{k}}^{2} \big)}$, where $\hat{\mu}_{\bvec{Y}} = \frac{1}{n} \sum_{i = 1}^{n} \bar{k}_{\alpha}(\bvec{y}_{i}, \cdot)$, $\hat{\mu}_{\bvec{X}} = \frac{1}{n} \sum_{i = 1}^{n} \bar{k}_{\alpha}(\bvec{x}_{i}, \cdot)$ are empirical mean embeddings of the observed and simulated raw data. Here $\bar{k}$ is another kernel with hyperparameters $\alpha$. This was also used in \gls{K2-ABC} \citep{park2016k2} and \gls{DR-ABC} \citep{mitrovic2016dr} to remove the requirement of summary statistics. 
		
	\subsection{Learning: Hyperparameter Learning with Marginal Kernel Means Likelihood}

			\begin{algorithm*}[tb]
				\caption{KELFI: Kernel Embedding Likelihood-Free Inference}
				\label{alg:kernel_means_posterior_super_sampling}
				\begin{algorithmic}[1]
					\STATE {\bfseries Input:} Data $\bvec{y}$, simulations $\{\bm{\theta}_{j}, \bvec{x}_{j}\}_{j = 1}^{m} \sim p(\bvec{x} | \bm{\theta}) \pi(\bm{\theta})$, query parameters $\{\bm{\theta}^{\star}_{r}\}_{r = 1}^{R}$, \gls{KML} hyperparameters $(\bm{\epsilon}, \bm{\beta}, \lambda)$, prior hyperparameters $(\bm{\mu}, \bm{\sigma})$ or samples $\{\tilde{\bm{\theta}}_{t}\}_{t = 1}^{T}$, number of samples $S$, kernel $\ell$ and $\epsilon$-kernel $\kappa$
					\STATE Compute $\bvec{v} \leftarrow (L + m \lambda I)^{-1} \bm{\kappa}_{\bm{\epsilon}}(\bvec{y})$ where $L \leftarrow \{\ell_{\bm{\beta}}(\bm{\theta}_{i}, \bm{\theta}_{j})\}_{i, j = 1}^{m}$ and $\bm{\kappa}_{\bm{\epsilon}}(\bvec{y}) \leftarrow \{\kappa_{\bm{\epsilon}}(\bvec{y}, \bvec{x}_{j})\}_{j = 1}^{m}$
					\STATE Compute $q(\bvec{y}) \leftarrow \bvec{v}^{T} \bm{\mu}_{\bm{\Theta}}$ where $\bm{\mu}_{\bm{\Theta}} \leftarrow \{\mu_{\bm{\Theta}}(\bm{\theta}_{j})\}_{j=1}^{m}$ using \eqref{eq:prior_embedding} or $\bm{\mu}_{\bm{\Theta}} \leftarrow \frac{1}{T} \tilde{L} \bvec{1}_{T}$ where $\tilde{L} \leftarrow \{\ell_{\bm{\beta}}(\bm{\theta}_{j}, \tilde{\bm{\theta}}_{t})\}_{j, t = 1}^{m, T}$
					\STATE Compute $H \leftarrow \{h(\bm{\theta}_{j}, \bm{\theta}^{\star}_{r}) \}_{j = 1, r = 1}^{m, R}$ using \eqref{eq:posterior_embedding_kernel}  or $H \leftarrow \frac{1}{T} \tilde{L} \tilde{L}^{\star}$ where $\tilde{L}^{\star} = \{\ell_{\bm{\beta}}(\tilde{\bm{\theta}}_{t}, \bm{\theta}^{\star}_{r})\}_{t, r = 1}^{T, R}$
					\STATE Compute posterior mean embedding $\bm{\mu} \leftarrow H^{T} \bvec{v} / q(\bvec{y}) \in \mathbb{R}^{R}$ and initialize $\bvec{a} \leftarrow \bvec{0} \in \mathbb{R}^{R}$
					\FOR{$s \in \{1, \dots, S\}$}
					\STATE Obtain super-sample $\hat{\bm{\theta}}_{s} \leftarrow \bm{\theta}^{\star}_{r^{\star}}$ where $r^{\star} \leftarrow \argmax_{r \in \{1, \dots, R\}} \mu_{r} - (a_{r} / s)$
					\STATE Update kernel sum $\bvec{a} \leftarrow \bvec{a} + \{\ell_{\bm{\beta}}(\bm{\theta}^{\star}_{r}, \hat{\bm{\theta}}_{s}) \}_{r = 1}^{R}$
					\ENDFOR
					\STATE {\bfseries Output:} Posterior super-samples $\{ \hat{\bm{\theta}}_{s} \}_{s = 1}^{S}$
					\vspace{-0.25em}
				\end{algorithmic}
			\end{algorithm*}
			
		We now propose a hyperparameter learning algorithm for our surrogate likelihood model. The main advantage of using an approximate surrogate likelihood surrogate model is that it readily provides a marginal surrogate likelihood quantity that lends itself to a hyperparameter learning algorithm. We define the \gls{MKML} as follows,
		\begin{equation}
			q(\bvec{y}) := \int_{\vartheta} q(\bvec{y} | \bm{\theta}) p(\bm{\theta}) d\bm{\theta} = \sum_{j = 1}^{m} v_{j}(\bvec{y}) \mu_{\bm{\Theta}}(\bm{\theta}_{j}),
		\label{eq:marginal_kernel_means_likelihood}
		\end{equation}
		where $\mu_{\bm{\Theta}} := \int_{\vartheta} \ell(\bm{\theta}, \cdot) p(\bm{\theta}) d\bm{\theta}$ is the mean embedding of $p_{\bm{\Theta}}$. If we choose $\ell$ to be an anisotropic Gaussian kernel with length scales $\bm{\beta} = \{\beta_{d}\}_{d = 1}^{D}$, then $\mu_{\bm{\Theta}}$ is closed-form for anisotropic Gaussian priors $p(\bm{\theta}) = \prod_{d = 1}^{D} \mathcal{N}(\theta_{d} | \mu_{d}, \sigma_{d}^{2})$. Let $\nu_{d}^{2} := \beta_{d}^{2} + \sigma_{d}^{2}$, then we have
		\begin{equation}
			\mu_{\bm{\Theta}}(\bm{\theta}) = \ell_{\bm{\nu}}(\bm{\theta}, \bm{\mu}) \prod_{d = 1}^{D} \frac{\beta_{d}}{\nu_{d}}.
		\label{eq:prior_embedding}
		\end{equation}
		Similar to the \gls{KML}, the \gls{MKML} converges at the same rate as the \gls{CME}. See \cref{thm:marginal_kernel_means_likelihood} for proof.
		\begin{theorem}
			\label{thm:marginal_kernel_means_likelihood_copy}
			Assume $\ell(\bm{\theta}, \cdot) \in \mathrm{image}(C_{\bm{\Theta} \bm{\Theta}})$. The \acrfull{MKML} $q(\bvec{y})$ converges to marginal likelihood $p_{\epsilon}(\bvec{y})$ uniformly at rate $O_{p}((m \lambda)^{-\frac{1}{2}} + \lambda^{\frac{1}{2}})$ as a function of $\bvec{y} \in \mathcal{Y}$.
		\end{theorem}
		Consequently, the \gls{MKML} $q(\bvec{y}) = q(\bvec{y}; \bm{\epsilon}, \bm{\beta}, \lambda)$ approximates the true marginal likelihood $p_{\epsilon}(\bvec{y})$ of the inference problem defined by our likelihood-prior pair. It is a function of the hyperparameters $(\bm{\epsilon}, \bm{\beta}, \lambda)$ of the $\epsilon$-kernel and \gls{KML} model. As $p_{\epsilon}(\bvec{y})$ is unavailable, we instead maximize the \gls{MKML} for hyperparameter learning. Furthermore, prior hyperparameters $\bm{\mu}$ and $\bm{\sigma}$ can also be included and learned jointly. Since the map $(\bm{\epsilon}, \bm{\beta}, \lambda) \mapsto q(\bvec{y}; \bm{\epsilon}, \bm{\beta}, \lambda)$ is differentiable, optimization can be done in an auto-differentiation environment. The learning objective to be optimized is computed in line 3 of \cref{alg:kernel_means_posterior_super_sampling}. Each automatic gradient update has complexity dominated by $O(m^{3})$ due to the Cholesky decomposition in line 2. However, since we are addressing scenarios where simulations are limited so that $m$ is small, this optimization is relatively fast.
		
		Importantly, if we use an anisotropic Gaussian density for the $\epsilon$-kernel $\kappa_{\bm{\epsilon}}$ where $\bm{\epsilon} = \{\epsilon_{i}\}_{i = 1}^{n}$ are the length scales corresponding to each summary statistic $\bvec{y} = \{y_{i}\}_{i = 1}^{n}$, we can perform \gls{ARD} to learn the relevance and usefulness of each summary statistic, where a small length scale indicate high relevance for that statistic. This is because $\bm{\epsilon}$ are also the length scales of the kernel $k$ which defines the \gls{RKHS} $\mathcal{H}_{k}$. Since the anistropic Gaussian kernel is learned, we also refer to it as an \gls{ARD} kernel. We can also learn the length scales $\bm{\beta} = \{\beta_{d}\}_{d = 1}^{D}$ for the kernel $\ell_{\bm{\beta}}$ on $\bm{\theta}$, although we found that it is more useful to let $\bm{\beta} = \beta_{0} \bm{\sigma}$ where $\bm{\sigma} = \{\sigma_{d}\}_{d = 1}^{D}$ are the standard deviations of the Gaussian prior. By doing this, we make better use of the scale differences within $\bm{\theta}$ from the prior, and let $\beta_{0}$ learn the overall scale that is most useful for the \gls{KML}.
		
		For general non-Gaussian kernels and priors, $\mu_{\bm{\Theta}}$ in \eqref{eq:marginal_kernel_means_likelihood} can be approximated using $T$ independent prior samples $\tilde{\bm{\theta}}_{t} \sim p(\bm{\theta})$, $t \in [T]$, as $\tilde{\mu}_{\bm{\Theta}} = \frac{1}{T} \sum_{t = 1}^{T} \ell(\tilde{\bm{\theta}}_{t}, \cdot)$.
		
		By formulating a learning objective directly for the inference problem, \gls{KELFI} provides a way to automatically tune $\epsilon$ and its own model hyperparameters. 
		
		
	\subsection{Inference: Kernel Means Posterior and Posterior Embedding Super-Sampling}
						
		We finally present an approach for posterior inference by super-sampling directly from the equivalent posterior mean embedding defined by the \gls{KML} model and the prior. Our approach begins by defining a surrogate density to approximate the posterior $p_{\epsilon}(\bm{\theta} | \bvec{y})$ in analogy to the Bayes' rule, $q(\bm{\theta} | \bvec{y}) := q(\bvec{y} | \bm{\theta}) p(\bm{\theta}) / q(\bvec{y})$. We refer to $q(\bm{\theta} | \bvec{y})$ as the \gls{KMP}. Importantly, $q(\bm{\theta} | \bvec{y})$ is unaffected even if $\kappa_{\epsilon}$ is unnormalized, so that $\epsilon$-kernels on distributions can be readily used. The \gls{KMP} has the following convergence properties. See \cref{thm:kernel_means_posterior} for proof. 
		\begin{theorem} 
			\label{thm:kernel_means_posterior_copy}
			Assume $\ell(\bm{\theta}, \cdot) \in \mathrm{image}(C_{\bm{\Theta} \bm{\Theta}})$ and that there exists $\delta > 0$ such that $q(\bvec{y}) \geq \delta$ for all $m \geq M$ where $M \in \mathbb{N}_{+}$. The \acrfull{KMP} $q(\bm{\theta} | \bvec{y})$ converges pointwise to the posterior $p_{\epsilon}(\bm{\theta} | \bvec{y})$ at rate $O_{p}((m \lambda)^{-\frac{1}{2}} + \lambda^{\frac{1}{2}})$ as a function of $\bm{\theta} \in \vartheta$ and $\bvec{y} \in \mathcal{Y}$. If $\sup_{\bm{\theta} \in \vartheta} p(\bm{\theta}) < \infty$ and $\sup_{\bm{\theta} \in \vartheta} p_{\epsilon}(\bvec{y} | \bm{\theta}) < \infty$, then the convergence is uniform in $\bm{\theta} \in \vartheta$. If $\sup_{\bvec{y} \in \mathcal{Y}} p_{\epsilon}(\bm{\theta} | \bvec{y}) < \infty$, then the convergence is uniform in $\bvec{y} \in \mathcal{Y}$.
		\end{theorem}
		Importantly, the requirement for a $\delta > 0$ such that $q(\bvec{y}) \geq \delta$ for all $m \geq M$ where $M \in \mathbb{N}_{+}$ provides an intuition for why high \gls{MKML} values are favorable for learning a good approximate posterior. This requirement is an reflection on the capability of the simulator to recreate the observations $\bvec{y}$ relative to the scale $\epsilon$. Intuitively, the more capable the simulator $p(\bvec{x} | \bm{\theta})$ is at generating simulations $\bvec{x}$ that is close to $\bvec{y}$ with respect to $\epsilon$, the higher $p_{\epsilon}(\bvec{y}) > 0$ will be relatively. Since \cref{thm:marginal_kernel_means_likelihood_copy} guarantees that, for large $m > M$, $q(\bvec{y})$ will be close to $p_{\epsilon}(\bvec{y})$, we have that $q(\bvec{y}) > 0$ for all large $m > M$ with increasing probability. In this situation, \cref{thm:kernel_means_posterior_copy} guarantees that the \gls{KMP} will converge to the posterior of interest. However, consider the case when the simulator is ill-designed to recreate $\bvec{y}$ such that the true marginal likelihood $p_{\epsilon}(\bvec{y}) \approx 0$ is small. As $q(\bvec{y})$ tends to $p_{\epsilon}(\bvec{y}) \approx 0$ due to \cref{thm:marginal_kernel_means_likelihood_copy}, it may struggle to always stay strictly positive even for large $m > M$ since it is stochastically converging to approximately zero. In this case, convergence is difficult since the simulator was ill-designed. However, by learning $\epsilon$ through maximizing $q(\bvec{y})$, we adapt the threshold $\epsilon$ to make $p_{\epsilon}(\bvec{y})$ as high as possible, leading to a more stable posterior $p_{\epsilon}(\bm{\theta} | \bvec{y})$ for the \gls{KMP} to converge to.
		
		We now define \gls{KMPE}, the mean embedding of the \gls{KMP}, as $\tilde{\mu}_{\bm{\Theta} | \bvec{Y} = \bvec{y}}(\bm{\theta}^{\star}) := \int_{\vartheta} \ell(\bm{\theta}, \bm{\theta}^{\star}) q(\bm{\theta} | \bvec{y}) d\bm{\theta}$. This becomes
		\begin{equation}
			\begin{aligned}
				\tilde{\mu}_{\bm{\Theta} | \bvec{Y} = \bvec{y}}(\bm{\theta}^{\star})
				= \frac{1}{q(\bvec{y})} \sum_{j = 1}^{m} v_{j}(\bvec{y}) h(\bm{\theta}_{j}, \bm{\theta}^{\star}),
			\end{aligned}
		\label{eq:kernel_means_posterior_embedding}
		\end{equation}
		where $h(\bm{\theta}, \bm{\theta}^{\star}) := \int_{\vartheta} \ell(\bm{\theta}, \tilde{\bm{\theta}}) \ell(\tilde{\bm{\theta}}, \bm{\theta}^{\star}) p(\tilde{\bm{\theta}}) d\tilde{\bm{\theta}}$. Importantly, since the \gls{KMPE} is constructed from the \gls{CME} used to form the \gls{KML}, it converges in \gls{RKHS} norm at the same rate. See \cref{thm:kernel_means_posterior_embedding} for proof. 
		\begin{theorem}
			\label{thm:kernel_means_posterior_embedding_copy}
			Assume $\ell(\bm{\theta}, \cdot) \in \mathrm{image}(C_{\bm{\Theta} \bm{\Theta}})$ and that there exists $\delta > 0$ such that $q(\bvec{y}) \geq \delta$ for all $m \geq M$ where $M \in \mathbb{N}_{+}$. The \acrfull{KMPE} $\tilde{\mu}_{\bm{\Theta} | \bvec{Y} = \bvec{y}}$ converges in \gls{RKHS} norm to the posterior mean embedding $\mu_{\bm{\Theta} | \bvec{Y} = \bvec{y}}$ at rate $O_{p}((m \lambda)^{-\frac{1}{2}} + \lambda^{\frac{1}{2}})$.
		\end{theorem}
		If we choose $\ell$ to be an anisotropic Gaussian kernel with length scales $\bm{\beta} = \{\beta_{d}\}_{d = 1}^{D}$, $h$ exhibits the following closed-form under anisotropic Gaussian priors, 
		\begin{equation}
			h(\bm{\theta}, \bm{\theta}^{\star}) = \prod_{d = 1}^{D} \frac{s_{d}}{\sigma_{d}} \exp{\bigg[ - \frac{1}{2 s_{d}^{2}} \Big( a_{d} - b_{d}^{2} \Big) \bigg]},
		\label{eq:posterior_embedding_kernel}
		\end{equation}
		where $a_{d} := (\theta_{d}^{2} + {\theta^{\star}_{d}}^{2} + \gamma_{d}^{2} \mu_{d}^{2}) / (2 + \gamma_{d}^{2})$, $b_{d} := (\theta_{d} + \theta^{\star}_{d} + \gamma_{d}^{2} \mu_{d}) /(2 + \gamma_{d}^{2})$, $\gamma_{d}^{2} := \beta_{d}^{2} / \sigma_{d}^{2}$ and $s_{d}^{-2} := 2 \beta_{d}^{-2} + \sigma_{d}^{-2}$. For general non-Gaussian kernels and priors, $h$ can be approximated as $\tilde{h}(\bm{\theta}, \bm{\theta}^{\star}) = \frac{1}{T} \sum_{t = 1}^{T} \ell(\bm{\theta}, \tilde{\bm{\theta}}_{t}) \ell(\tilde{\bm{\theta}}_{t}, \bm{\theta}^{\star})$.
		
		The \gls{KMP} $q(\cdot | \bvec{y})$ is bounded and normalized but potentially non-positive. Consequently, it can be seen as a surrogate density corresponding to a signed measure. This suggests that the map $q(\cdot | \bvec{y}) \mapsto \tilde{\mu}_{\bm{\Theta} | \bvec{Y} = \bvec{y}}$ is injective for characteristic kernels $\ell$, analogous to mean embeddings \citep{sriperumbudur2011universality}. Furthermore, as the integral \eqref{eq:kernel_means_posterior_embedding} is a linear operator on $\ell(\bm{\theta}^{\star}, \cdot)$, the surrogate posterior mean embedding $\tilde{\mu}_{\bm{\Theta} | \bvec{Y} = \bvec{y}} \in \mathcal{H}_{\ell}$ is in the \gls{RKHS} of $\ell$. With a surrogate embedding that is injective to our surrogate posterior and in the \gls{RKHS}, we can apply kernel herding \citep{chen2010super} on $\tilde{\mu}_{\bm{\Theta} | \bvec{Y} = \bvec{y}}$ \eqref{eq:kernel_means_posterior_embedding} using kernel $\ell$ to obtain $S$ super-samples $\{ \hat{\bm{\theta}}_{s} \}_{s = 1}^{S}$ from the surrogate density $q(\bm{\theta} | \bvec{y})$. That is, for each $s \in [S]$, the samples are obtained by
		\begin{equation}
			\hat{\bm{\theta}}_{s} = \argmax_{\bm{\theta} \in \vartheta} \tilde{\mu}_{\bm{\Theta} | \bvec{Y} = \bvec{y}}(\bm{\theta}) - \frac{1}{s} \sum_{s' = 1}^{s - 1} \ell(\hat{\bm{\theta}}_{s'}, \bm{\theta}).
		\end{equation}
		The inference algorithm is presented in \cref{alg:kernel_means_posterior_super_sampling}. 
			
\section{Related Work} 
	
			\begin{figure*}[t]
				\centering
				\includegraphics[width=0.49\linewidth]{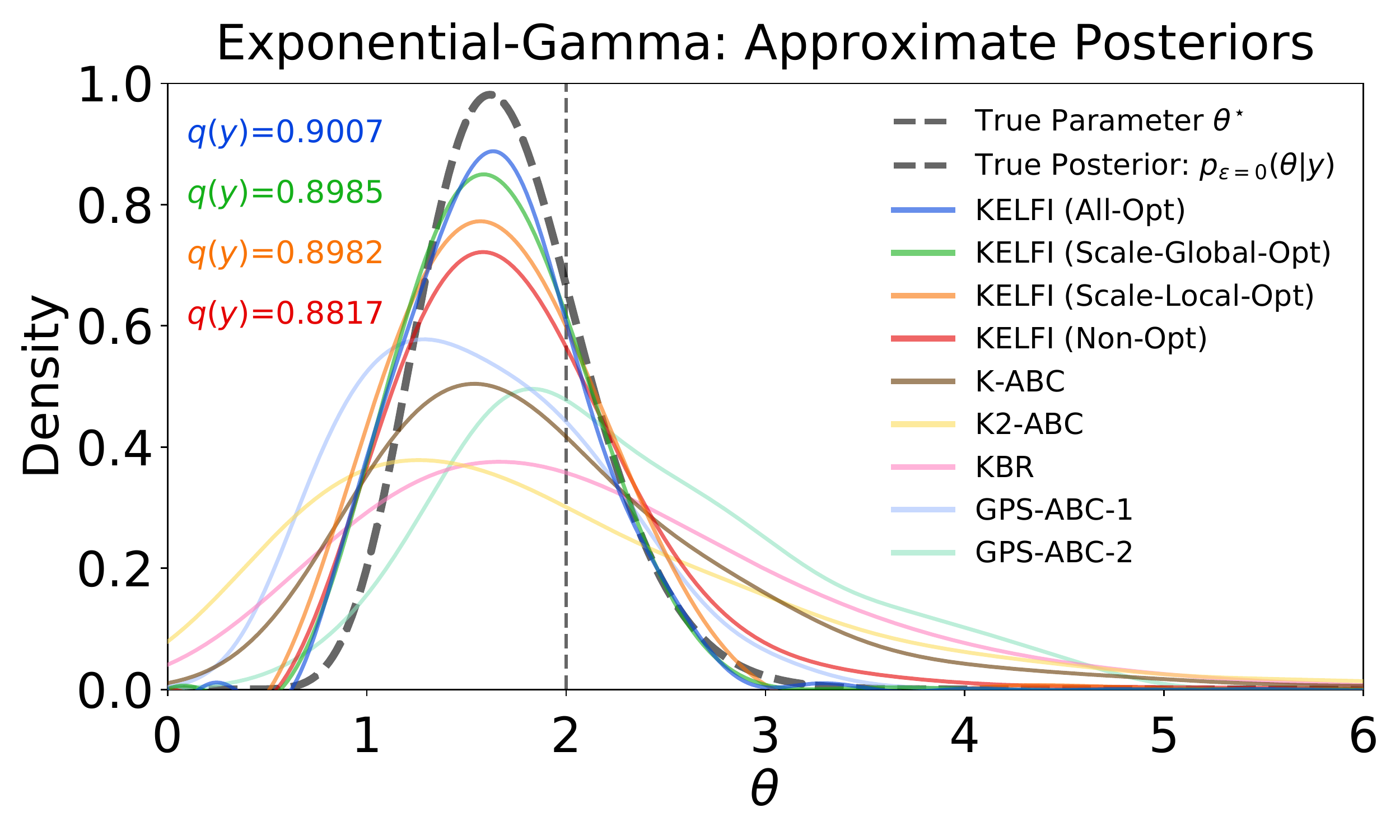}
				\includegraphics[width=0.49\linewidth]{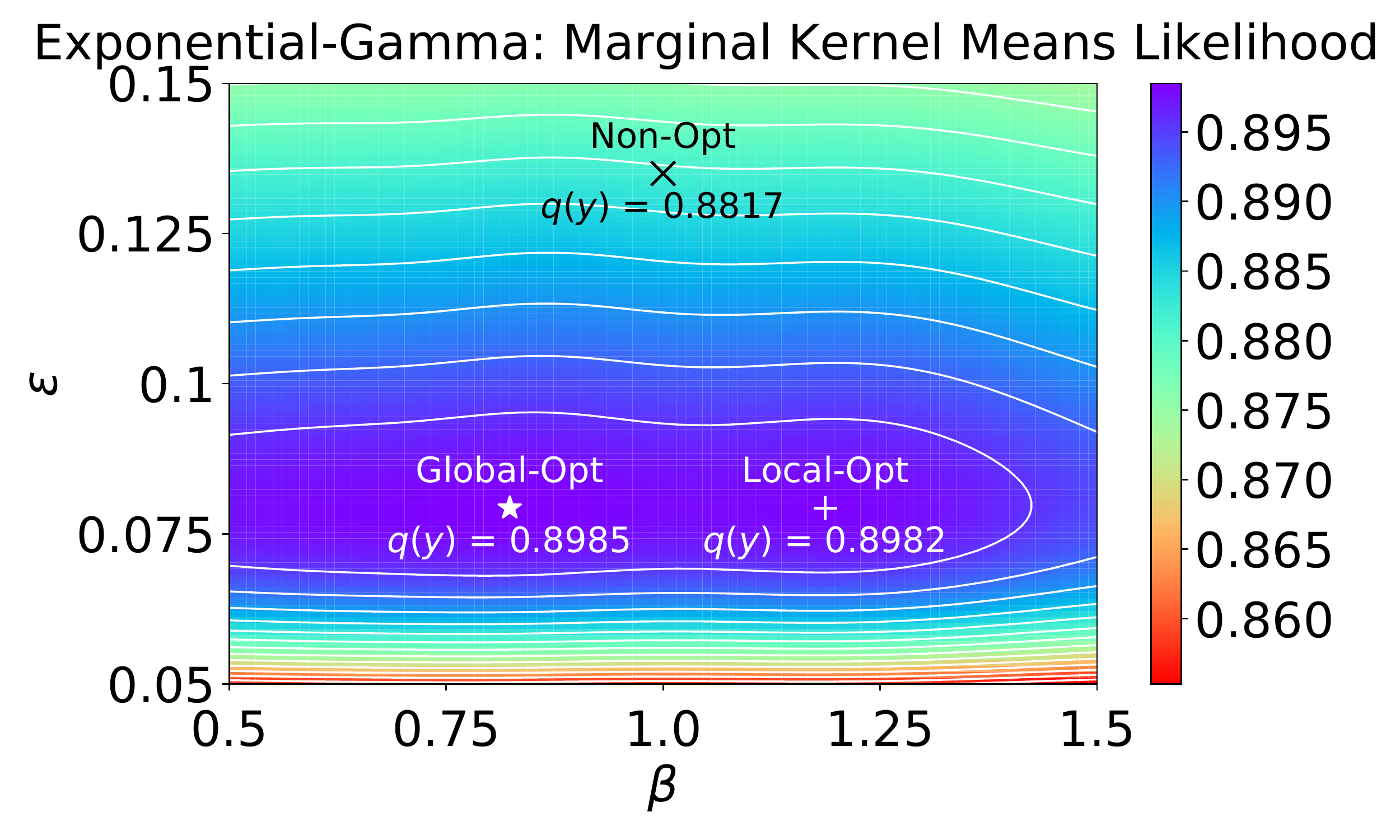}
				\vspace{-1.5em} 
				\caption{{(\bf{Left})} Comparison of approximate posteriors obtained from surrogate methods on the toy exponential-gamma problem. {(\bf{Right})} The corresponding \gls{MKML} surface $q(\bvec{y})$ as a function of $(\epsilon, \beta)$ with $\lambda = 10^{-3} \beta$.}
			\label{fig:toy}
			\end{figure*}
			
		The simplest \gls{ABC} algorithm is arguably the \gls{REJ-ABC} sampler \citep{pritchard1999population}. It posits a set of prior parameters and rejects those whose simulations do not match the observations within a fixed threshold $\epsilon > 0$ under a distance measure. 
		
		Instead of sampling from the prior, \gls{MCMC}-\gls{ABC} and \gls{SMC-ABC} sample from proposal distributions iteratively and carefully accepts or discards each proposal stochastically based on approximate likelihood ratios \citep{sisson2007sequential, marjoram2003markov}. They can however suffer from slow mixing, where it is difficult to escape a lucky sample with a high likelihood. They also do not leverage likelihood smoothness and thus require new simulations every iteration, which are then discarded and may still not result in an accepted sample. 

		Another branch of study include \gls{SVI} approaches to \gls{ABC}, which treats the likelihood approximation as another source of stochasticity in the stochastic gradient. This includes \acrshort{AV-ABC} \citep{moreno2016automatic}, \acrshort{VBIL} \citep{tran2017variational}, and \acrshort{VBSL} \citep{ong2018variational}. In contrast, \gls{LFVI} \citep{tran2017hierarchical} uses density ratio estimation to approximate the variational objective, emphasizing inference on local latent variables. Nevertheless, \gls{SVI} approaches posit parametric approximations that may have asymptotic bias. 
		
		Kernel-based approaches that leverage likelihood smoothness have been studied recently to reduce simulation requirements. The philosophy is that simulations of close-by parameters are informative, thus past results should not be discarded but remembered, even if this introduces model bias.
		\Gls{K-ABC} \citep{nakagome2013kernel}, \gls{KR-ABC} \citep{kajihara2018kernel}, and \gls{KBR} \citep{fukumizu2013kernel} also employ \glspl{CME} to reduce simulation requirements. They differ to \gls{KELFI} in the three aspects of model, learning, and inference. (Model) While they build posterior mean embeddings directly, \gls{KELFI} builds likelihood surrogates first and make use of the full prior density to further leverage prior information before building posterior surrogates, which are then embedded into closed-form posterior mean embeddings. In contrast, the prior only appears as samples from $p(\bm{\theta})$ in \gls{K-ABC}, \gls{KR-ABC}, and \gls{KBR}. This both limits the prior knowledge leveraged and prohibit the use of proposal prior samples. (Learning) \gls{KELFI} crucially addresses hyperparameter learning in reference to the inference problem directly which was not straightforward previously. (Inference) \Gls{K-ABC} and \gls{KBR} primarily infer posterior expectations, while \gls{KR-ABC} produce point estimates. Instead, we design a posterior sampling algorithm, which subsumes inferring posterior expectation. We further provide approximate posterior density \gls{KMP}, which can both produce point estimates and quantify uncertainty.
		
		As a consequence of \cref{thm:kernel_means_posterior_embedding_copy}, the \gls{KMPE} converges at rate $O_{p}(m^{-\frac{1}{4}})$ in \gls{RKHS} norm if the regularization hyperparameter $\lambda$ is chosen to decay at rate $O_{p}(m^{-\frac{1}{2}})$. Notably, this is faster than the convergence rate of \gls{KBR} at $O_{p}(m^{-\frac{8}{27} \alpha})$ where $0 < \alpha \leq \frac{1}{2}$, which also requires other assumptions on the cross-covariance operators and for its two regularization hyperparameters to be decayed appropriately \citep{fukumizu2013kernel}.
		
		Finally, we highlight that hyperparameter learning is a crucial aspect and differentiator of \gls{KELFI}. This is especially true for learning $\epsilon$, which tunes the critical balance between an accurate posterior $p_{\epsilon}(\bm{\theta} | \bvec{y}) \approx p_{0}(\bm{\theta} | \bvec{y})$ with small $\epsilon$ requiring high numbers of simulation calls, or a less accurate posterior with large $\epsilon$ relaxing the number of simulations required. This has been a challenging issue to address in the \gls{ABC} literature in reference to the inference problem, even though its selection is often pivotal to the performance of the algorithm.
		
		In the \gls{GP} literature, hyperparameter learning through maximum marginal likelihood plays an important role in the success of a \gls{GPR}. \Gls{GPS-ABC} \citep{meeds2014gps} and \Gls{GPA-ABC} \citep{wilkinson2014accelerating} model the summary statistics surface and log likelihood surface respectively via a \gls{GP} surrogate. In contrast, the \gls{KML} model is equivalent to placing a \gls{GP} surrogate on the likelihood surface itself. This removes the assumption that summary statistics are independent and Gaussian distributed as in \gls{GPS-ABC}. 
		Importantly, while \gls{GPS-ABC} and \gls{GPA-ABC} apply the \gls{GP} marginal likelihood to learn their surrogate hyperparameters, it cannot learn $\epsilon$ or other hyperparameters since they are not part of the surrogate. This is because both approaches maximize the marginal likelihood for the \gls{GPR} problem on the their respective target surfaces, but not the marginal likelihood for the overall inference problem, thus excluding other hyperparameters in the process.
			
\section{Experiments}

			\begin{figure*}[t]
				\centering
				\includegraphics[width=\linewidth]{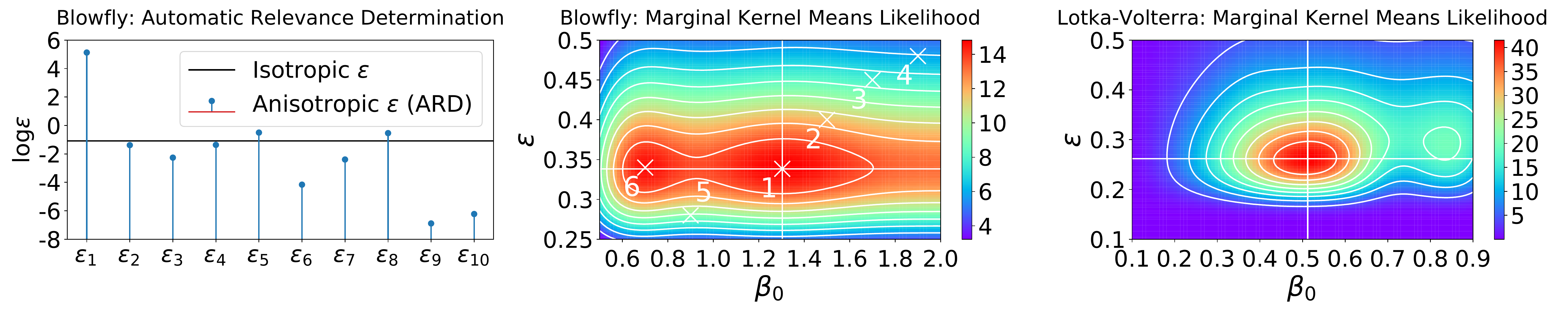}
				\vspace{-2.5em} 
				\caption{{(\bf{Left})} Blowfly: \gls{ARD} on $\bm{\epsilon}$ for 10 summary statistics. {(\bf{Mid. \& Right})} The \gls{MKML} surface ($\times 10^{5}$) as a function of $(\epsilon, \beta_{0})$ for fixed $\lambda = 10^{-3} \beta_{0}$ where $\bm{\beta} = \beta_{0} \bm{\sigma}$. White intersection indicate optimum. For Blowfly, the NMSE (in \%) for the indicated hyperparameter choices are: $(1) 0.72 \pm 0.02, (2) 1.10 \pm 0.01, (3) 2.07 \pm 0.01, (4) 2.15 \pm 0.02, (5) 1.11 \pm 0.02, (6) 1.11 \pm 0.03$. At $(\epsilon, \beta_{0}) = (10, 10)$ (outside the plot) the NMSE (in \%) is $6.28 \pm 0.03$.}
				\label{fig:learning}
			\end{figure*}
			
	The goal of the experiments is to demonstrate the inference accuracy of \gls{KELFI} under limited simulation budget and the effectiveness of \gls{MKML} hyperparameter learning. We begin with isotropic $\epsilon$ and anisotropic $\bm{\beta} = \beta_{0} \bm{\sigma}$, and learn $(\epsilon, \beta_{0})$ by maximizing the \gls{MKML} \eqref{eq:marginal_kernel_means_likelihood} while keeping $\lambda = 10^{-3} \beta_{0}$ fixed for simplicity. 
	
	\subsection{Toy Problem: Exponential-Gamma}
			
		The toy exponential-gamma problem is a standard benchmark for likelihood-free inference, since the true posterior $p_{\epsilon}(\bm{\theta} | \bvec{y})$ is known and tractable even for $\epsilon = 0$. 
		
		To stress-test each method, we compare inference accuracy under very limited simulations of $m = 100$. We focus on comparing surrogate approaches, since other methods such as \gls{REJ-ABC}, \gls{MCMC}-\gls{ABC}, \gls{SL-ABC}, and \gls{ASL-ABC} have reported simulation requirements several orders higher than $100$ on this problem \citep{meeds2014gps}. We use datasets of $n = 15$ for both observations and simulations, with their sample means as the summary statistic.
		
		For \gls{GPS-ABC} only we set a simulation budget of $m \leq 200$ and run it until $10000$ posterior samples are generated. The hyperparameters of the \gls{GP} surrogate itself are learned by maximizing the marginal likelihood of the \gls{GPR} \citep{rasmussen2006gaussian}. For the remaining hyperparameters that are not part of the surrogate, several configurations are compared and the results of the best two are shown, which used $m = 130$ and $m = 197$ simulations. For \gls{K-ABC}, \gls{K2-ABC}, and \gls{KBR}, we use the median heuristic to set their length scale hyperparameters and manually search for the most appropriate regularization hyperparameters. We use \gls{KDE} to visualize the posterior density from the unweighted samples of \gls{GPS-ABC} and normalized weighted samples of \gls{K-ABC}, \gls{K2-ABC}, and \gls{KBR} in \cref{fig:toy} (left). 
		
		
		For \gls{KELFI}, we show the \glspl{KMP} directly in \cref{fig:toy} (left). We first demonstrate the case when all hyperparameters $(\epsilon, \beta, \lambda)$ are learned (All-Opt). To enable visualization in 2D, we also present the case when the regularization hyperparameter $\lambda$ is set to $10^{-3} \beta$ and only length scale hyperparameters $(\epsilon, \beta)$ are learned. In this case, we show \glspl{KMP} under globally optimal (Scale-Global-Opt), locally optimal (Scale-Local-Opt), and arbitrarily chosen hyperparameters (Non-Opt). The corresponding \gls{MKML} surface is shown in \cref{fig:toy} (right).
		
		In \cref{fig:toy} we compare approximate posteriors from each algorithm against the true posterior $p_{\epsilon = 0}(\bm{\theta} | \bvec{y})$. While $\epsilon = 0$ for $p_{\epsilon = 0}(\bm{\theta} | \bvec{y})$, with only $100$ simulations $\epsilon > 0$ is required for most \gls{LFI} methods. Furthermore, except for \gls{K2-ABC}, they only make use of summary statistics without further knowledge of the dataset size $n$. Consequently, most \gls{LFI} methods produce approximations wider than $p_{\epsilon=0}(\bm{\theta} | \bvec{y})$. Intuitively, there is not enough simulations and thus information to justify a more confident and peaked posterior. Nevertheless, by learning hyperparameters under the \gls{MKML}, \gls{KELFI} determines an appropriate scale $\epsilon$ for $100$ simulations. As a result, \glspl{KMP} are the closest to the true posterior $p_{\epsilon=0}(\bm{\theta} | \bvec{y})$, with higher \gls{MKML} $q(\bvec{y})$ leading to more accurate \glspl{KMP} $q(\bm{\theta} | \bvec{y})$. This demonstrates the effectiveness of \gls{MKML} as a hyperparameter learning objective for improving inference accuracy. In contrast, the two instances of \gls{GPS-ABC} reveals that varying hyperparameters lead to significant changes in the resulting approximate posterior, yet without a similar objective like \gls{MKML} it is unclear which one to use without ground truth. This is further emphasized by the wider posterior approximations obtained from \gls{K-ABC}, \gls{K2-ABC}, and \gls{KBR}, which use the median heuristic to set hyperparameters. This is often sub-optimal since the heuristic makes no reference to the inference problem.
	
	\subsection{Chaotic Ecological Systems: Blowfly}
			
			\begin{figure*}[t]
				\centering
				\includegraphics[width=0.63\linewidth]{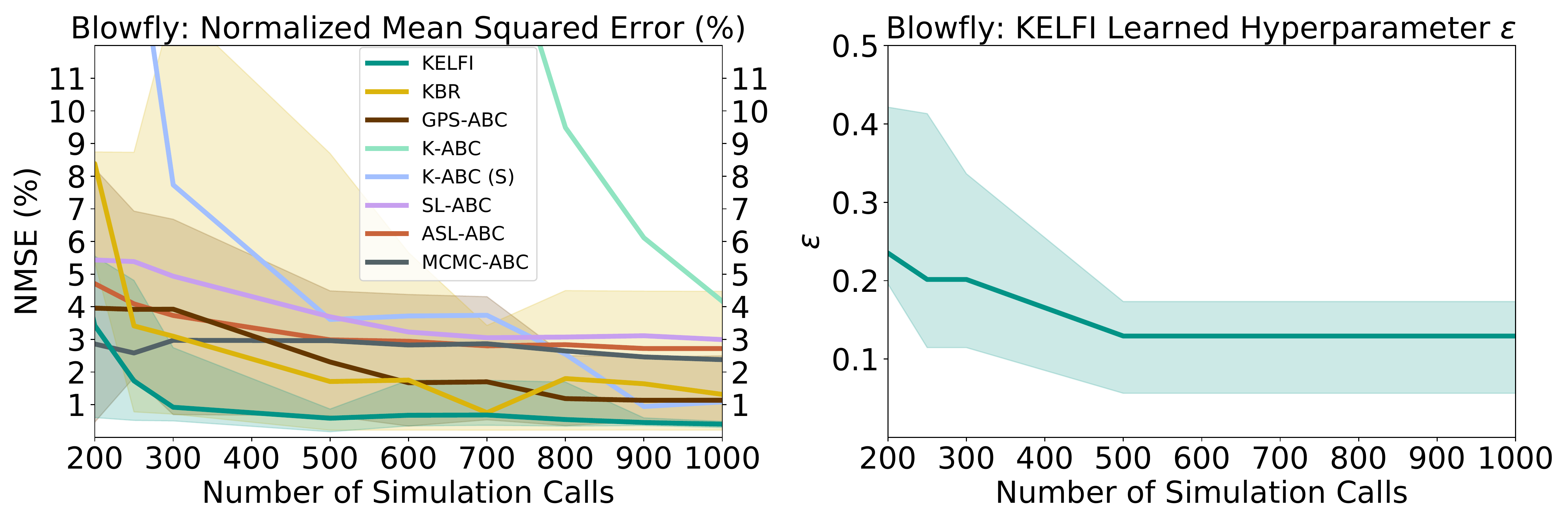}
				\includegraphics[width=0.35\linewidth]{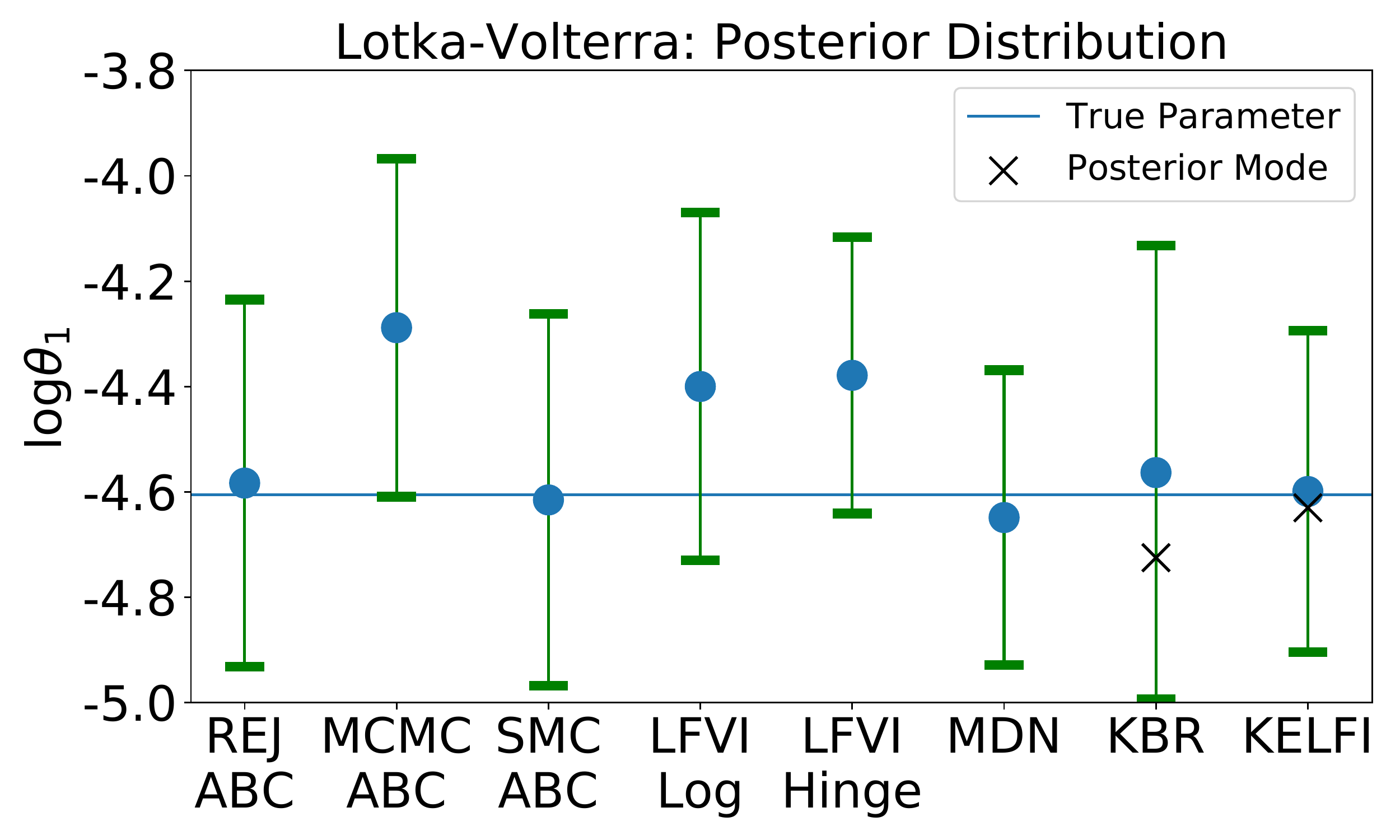}
				\vspace{-1em} 
				\caption{{(\bf{Left})} Blowfly: Average NMSE (in \%) under posteriors against simulation calls. Shaded regions show NMSE variability for KELFI, KBR, and GPS-ABC. {(\bf{Mid.})} Blowfly: Learned $\epsilon$ value under maximum \gls{MKML}. {(\bf{Right})} Lotka-Volterra: The middle 95\% credible interval of the marginal posterior distribution of $\log{\theta_{1}}$.}
				\label{fig:inference}
			\end{figure*}
		
		The Blowfly simulator describes the complex population dynamics of adult blowflies. Across a range of parameters it exhibits chaotic behavior that have distinct discrepancies from real observations, resulting in a challenging inference problem. We follow the setup of \cite{wood2010statistical}. There are 6 model parameters from which the simulator generates a time series of 180 data points that is then summarized into 10 statistics as described in \cite{meeds2014gps}, \cite{moreno2016automatic}, and \cite{park2016k2}. We similarly place a broad diagonal Gaussian prior on log parameters. 
		
		The standard Blowfly problem has no ground truth parameters, only a set of observations. We therefore measure inference accuracy by considering \glspl{MSE} between statistics generated using the posterior and the observed statistics. We normalize the \gls{MSE} of each statistic by the corresponding \gls{MSE} achieved under the prior, and average across the 10 statistics into a final \gls{NMSE}. As simulations are expensive, in \cref{fig:inference} (left) we record average \gls{NMSE} against simulations used to understand inference efficiency. Each method is repeated $10$ times with randomized simulations before their \gls{NMSE} is averaged. \Cref{sec:blowfly} provides further details.
		
		As new simulations become available, we relearn and update the hyperparameters for \gls{KELFI} by maximizing the \gls{MKML}. \Cref{fig:learning} (center) shows an instance of the \gls{MKML} surface used to learn the hyperparameters for \gls{KELFI} when using $m = 280$ simulations. For \gls{KBR} and \gls{K-ABC} we update hyperparameters by the median length heuristic. For \gls{K-ABC} we also report the case where the heurstic is scaled by a constant denoted with (S), which achieved significantly better accuracy and confirms that the heuristic is often sub-optimal.
		
		Overall, the top three performers are \gls{KELFI}, \gls{KBR}, and \gls{GPS-ABC}. Across a range of simulation calls, \gls{KELFI} achieves the lowest error. It is also the only method that achieved less than $1\%$ average \gls{NMSE} within $1000$ simulations and  achieves this as early as $300$ simulations. The most competitive methods to \gls{KELFI} are \gls{KBR} and \gls{GPS-ABC}. For these three methods, we also show their variability from best to worst case \glspl{NMSE} out of the $10$ repeats to visualize their sensitivity to the stochasticity in randomized simulations. This reveals that \gls{KELFI} is a stable outperformer with comparatively less variability across randomized runs.
		
		We proceed to demonstrate and emphasize the effectiveness and suitability of \gls{MKML} as a hyperparameter learning objective, using the case with $280$ simulations as an example. \Cref{fig:learning} (center) illustrates that hyperparameters with a higher \gls{MKML} \eqref{eq:marginal_kernel_means_likelihood} result in lower \gls{NMSE} consistently. Notably, even with suboptimal hyperparameter choices, \gls{KELFI} still achieves competitive average \gls{NMSE} scores of less than 2.2\%. At $280$ simulations, the next best average \gls{NMSE} score is almost 3\% by \gls{MCMC}-\gls{ABC} as shown in \cref{fig:inference} (left).
		
		\Cref{fig:inference} (center) suggests that learning the scale $\epsilon$ under \gls{MKML} reveals an automatic decay schedule which does not have to be set a-priori. As $\epsilon$ controls the scale within which discrepancies between simulations and observations are measured, it is expected that this scale decays as more simulation data is available. Without the \gls{MKML}, both the initialization of $\epsilon$ and its decay schedule are not straight forward to determine. 
		
		In \cref{fig:learning} (left), we show that we can perform \gls{ARD} on the \gls{ABC} $\epsilon$-kernel $\kappa_{\bm{\epsilon}}$, and hence the kernel $k_{\bm{\epsilon}}$, by using a different $\epsilon_{i}$ for each of the $10$ statistics. We do this by initializing each $\epsilon_{i}$ to the isotropic solution in \cref{fig:learning} (center) and further optimize the \gls{MKML} to learn all $\epsilon_{i}$ jointly. In particular, the first summary statistic describes the average log population numbers nears its troughs (first quartile), and is determined to be comparatively irrelevant (high $\epsilon_{i}$). Meanwhile, the last two statistics describe the number of peaks at two thresholds, and are determined to be comparatively relevant (low $\epsilon_{i}$). This agrees with the intuition that Blowfly population dynamics are highly characterized by its peaks, instead of i its troughs \citep{wood2010statistical}.
		

	\subsection{Predator-Prey Dynamics: Lotka-Volterra}
	
		The Lotka-Volterra simulator describes the time evolution of the populations within a predator-prey system. Only for a small set of parameters does the model simulate a realistic scenario with oscillatory behavior, making the inference task formidably challenging. We follow the exact setup as described in \cite{papamakarios2016fast}. There are 4 parameters and 9 normalized summary statistics. We place the same uniform prior on the log parameters and use the same ground truth parameters. After performing inference on all four parameters, we show in \cref{fig:inference} (right) the marginal posterior distribution for $\log{\theta_{1}}$.
		
		\Gls{KELFI} achieves competitive performance using only $2500$ simulations, with both posterior mean and mode close to the true value. The \gls{MKML} for hyperparameter learning is shown in \cref{fig:learning} (right). Posterior mode is obtained by maximizing the \gls{KMP}. Meanwhile, the three \gls{ABC} methods used up to $100000$ simulations. While confident, \gls{LFVI} \citep{tran2017hierarchical} tends to have a biased posterior mean. For direct comparison, both \gls{KELFI} and \gls{MDN} \citep{papamakarios2016fast} use the original prior as the proposal prior. \gls{KELFI} achieves slightly higher accuracy than \gls{MDN} which used $10000$ simulations, $4$ times that used for \gls{KELFI}. Finally, we also similarly use $2500$ simulations for \gls{KBR}. With the same number of simulations, \gls{KELFI} achieves higher accuracy in both mean and mode with higher confidence.
		

\vspace{-0.5em}
\section{Conclusion}

	\Gls{KELFI} provides a holistic framework for automatic likelihood-free inference. It is a stable outperformer compared to state-of-the-art methods, while producing interpretable automatic relevance determination of summary statistics and automatic decay schedules for $\epsilon$. By optimizing an approximate Bayesian marginal likelihood, it automatically learns and adapts hyperparameters including the $\epsilon$-kernel to improve inference accuracy when limited simulations are available. 

\newpage
\bibliographystyle{apalike}
\bibliography{references}

\begin{thebibliography}{}

\bibitem[Andrieu et~al., 2009]{andrieu2009pseudo}
Andrieu, C., Roberts, G.~O., et~al. (2009).
\newblock The pseudo-marginal approach for efficient monte carlo computations.
\newblock {\em The Annals of Statistics}, 37(2):697--725.

\bibitem[Beaumont, 2010]{beaumont2010approximate}
Beaumont, M.~A. (2010).
\newblock Approximate bayesian computation in evolution and ecology.
\newblock {\em Annual review of ecology, evolution, and systematics},
  41:379--406.

\bibitem[Carmeli et~al., 2010]{carmeli2010vector}
Carmeli, C., De~Vito, E., Toigo, A., and Umanit{\'a}, V. (2010).
\newblock Vector valued reproducing kernel hilbert spaces and universality.
\newblock {\em Analysis and Applications}, 8(01):19--61.

\bibitem[Chen et~al., 2010]{chen2010super}
Chen, Y., Welling, M., and Smola, A. (2010).
\newblock {Super-samples from kernel herding}.
\newblock In {\em {The Twenty-Sixth Conference Annual Conference on Uncertainty
  in Artificial Intelligence (UAI-10)}}, pages 109--116. AUAI Press.

\bibitem[Flaxman et~al., 2016]{flaxman2016bayesian}
Flaxman, S., Sejdinovic, D., Cunningham, J.~P., and Filippi, S. (2016).
\newblock {Bayesian learning of kernel embeddings}.
\newblock In {\em {Proceedings of the Thirty-Second Conference on Uncertainty
  in Artificial Intelligence}}, pages 182--191. AUAI Press.

\bibitem[Fukumizu et~al., 2004]{fukumizu2004dimensionality}
Fukumizu, K., Bach, F.~R., and Jordan, M.~I. (2004).
\newblock {Dimensionality reduction for supervised learning with reproducing
  kernel Hilbert spaces}.
\newblock {\em Journal of Machine Learning Research}, 5(Jan):73--99.

\bibitem[Fukumizu et~al., 2013]{fukumizu2013kernel}
Fukumizu, K., Song, L., and Gretton, A. (2013).
\newblock {Kernel Bayes' rule: Bayesian inference with positive definite
  kernels.}
\newblock {\em Journal of Machine Learning Research}, 14(1):3753--3783.

\bibitem[Gr{\"u}new{\"a}lder et~al., 2012]{grunewalder2012conditional}
Gr{\"u}new{\"a}lder, S., Lever, G., Baldassarre, L., Patterson, S., Gretton,
  A., and Pontil, M. (2012).
\newblock Conditional mean embeddings as regressors.
\newblock In {\em Proceedings of the 29th International Conference on Machine
  Learning, ICML 2012}, volume~2, pages 1823--1830.

\bibitem[Kajihara et~al., 2018]{kajihara2018kernel}
Kajihara, T., Kanagawa, M., Yamazaki, K., and Fukumizu, K. (2018).
\newblock Kernel recursive abc: Point estimation with intractable likelihood.
\newblock In {\em International Conference on Machine Learning}, pages
  2405--2414.

\bibitem[Kanagawa et~al., 2016]{kanagawa2016filtering}
Kanagawa, M., Nishiyama, Y., Gretton, A., and Fukumizu, K. (2016).
\newblock {Filtering with state-observation examples via kernel monte carlo
  filter}.
\newblock {\em Neural computation}, 28(2):382--444.

\bibitem[Marin et~al., 2012]{marin2012approximate}
Marin, J.-M., Pudlo, P., Robert, C.~P., and Ryder, R.~J. (2012).
\newblock {Approximate Bayesian computational methods}.
\newblock {\em Statistics and Computing}, 22(6):1167--1180.

\bibitem[Marjoram et~al., 2003]{marjoram2003markov}
Marjoram, P., Molitor, J., Plagnol, V., and Tavar{\'e}, S. (2003).
\newblock Markov chain monte carlo without likelihoods.
\newblock {\em Proceedings of the National Academy of Sciences},
  100(26):15324--15328.

\bibitem[Meeds and Welling, 2014]{meeds2014gps}
Meeds, E. and Welling, M. (2014).
\newblock Gps-abc: Gaussian process surrogate approximate bayesian computation.
\newblock In {\em Proceedings of the Thirtieth Conference on Uncertainty in
  Artificial Intelligence}, pages 593--602. AUAI Press.

\bibitem[Mitrovic et~al., 2016]{mitrovic2016dr}
Mitrovic, J., Sejdinovic, D., and Teh, Y.~W. (2016).
\newblock {DR-ABC: Approximate Bayesian Computation with kernel-based
  distribution regression}.

\bibitem[Moreno et~al., 2016]{moreno2016automatic}
Moreno, A., Adel, T., Meeds, E., Rehg, J.~M., and Welling, M. (2016).
\newblock Automatic variational abc.
\newblock {\em stat}, 1050:28.

\bibitem[Muandet et~al., 2017]{muandet2017kernel}
Muandet, K., Fukumizu, K., Sriperumbudur, B., Sch{\"o}lkopf, B., et~al. (2017).
\newblock Kernel mean embedding of distributions: A review and beyond.
\newblock {\em Foundations and Trends{\textregistered} in Machine Learning},
  10(1-2):1--141.

\bibitem[Nakagome et~al., 2013]{nakagome2013kernel}
Nakagome, S., Fukumizu, K., and Mano, S. (2013).
\newblock {Kernel approximate Bayesian computation in population genetic
  inferences}.
\newblock {\em Statistical applications in genetics and molecular biology},
  12(6):667--678.

\bibitem[Ong et~al., 2018]{ong2018variational}
Ong, V.~M., Nott, D.~J., Tran, M.-N., Sisson, S.~A., and Drovandi, C.~C.
  (2018).
\newblock {Variational Bayes with synthetic likelihood}.
\newblock {\em Statistics and Computing}, 28(4):971--988.

\bibitem[Papamakarios and Murray, 2016]{papamakarios2016fast}
Papamakarios, G. and Murray, I. (2016).
\newblock Fast $\varepsilon$-free inference of simulation models with bayesian
  conditional density estimation.
\newblock In {\em Advances in Neural Information Processing Systems}, pages
  1028--1036.

\bibitem[Park et~al., 2016]{park2016k2}
Park, M., Jitkrittum, W., and Sejdinovic, D. (2016).
\newblock {K2-ABC: Approximate Bayesian computation with kernel embeddings}.

\bibitem[Price et~al., 2017]{price2017bayesian}
Price, L.~F., Drovandi, C.~C., Lee, A., and Nott, D.~J. (2017).
\newblock Bayesian synthetic likelihood.
\newblock {\em Journal of Computational and Graphical Statistics}, pages 1--11.

\bibitem[Pritchard et~al., 1999]{pritchard1999population}
Pritchard, J.~K., Seielstad, M.~T., Perez-Lezaun, A., and Feldman, M.~W.
  (1999).
\newblock Population growth of human y chromosomes: a study of y chromosome
  microsatellites.
\newblock {\em Molecular biology and evolution}, 16(12):1791--1798.

\bibitem[Rasmussen and Williams, 2006]{rasmussen2006gaussian}
Rasmussen, C.~E. and Williams, C. K.~I. (2006).
\newblock {\em {Gaussian processes for machine learning}}.
\newblock The MIT Press.

\bibitem[Sisson et~al., 2007]{sisson2007sequential}
Sisson, S.~A., Fan, Y., and Tanaka, M.~M. (2007).
\newblock Sequential monte carlo without likelihoods.
\newblock {\em Proceedings of the National Academy of Sciences},
  104(6):1760--1765.

\bibitem[Snoek et~al., 2012]{snoek2012practical}
Snoek, J., Larochelle, H., and Adams, R.~P. (2012).
\newblock Practical bayesian optimization of machine learning algorithms.
\newblock In {\em Advances in neural information processing systems}, pages
  2951--2959.

\bibitem[Song et~al., 2013]{song2013kernel}
Song, L., Fukumizu, K., and Gretton, A. (2013).
\newblock {Kernel embeddings of conditional distributions: A unified kernel
  framework for nonparametric inference in graphical models}.
\newblock {\em IEEE Signal Processing Magazine}, 30(4):98--111.

\bibitem[Song et~al., 2009]{song2009hilbert}
Song, L., Huang, J., Smola, A., and Fukumizu, K. (2009).
\newblock {Hilbert space embeddings of conditional distributions with
  applications to dynamical systems}.
\newblock In {\em {Proceedings of the 26th Annual International Conference on
  Machine Learning}}, pages 961--968. ACM.

\bibitem[Sriperumbudur et~al., 2010]{sriperumbudur2010relation}
Sriperumbudur, B.~K., Fukumizu, K., and Lanckriet, G.~R. (2010).
\newblock {On the relation between universality, characteristic kernels and
  RKHS embedding of measures.}
\newblock In {\em {AISTATS}}, pages 773--780.

\bibitem[Sriperumbudur et~al., 2011]{sriperumbudur2011universality}
Sriperumbudur, B.~K., Fukumizu, K., and Lanckriet, G.~R. (2011).
\newblock {Universality, characteristic kernels and RKHS embedding of
  measures}.
\newblock {\em Journal of Machine Learning Research}, 12(Jul):2389--2410.

\bibitem[Toni et~al., 2009]{toni2009approximate}
Toni, T., Welch, D., Strelkowa, N., Ipsen, A., and Stumpf, M.~P. (2009).
\newblock Approximate bayesian computation scheme for parameter inference and
  model selection in dynamical systems.
\newblock {\em Journal of the Royal Society Interface}, 6(31):187--202.

\bibitem[Tran et~al., 2017a]{tran2017hierarchical}
Tran, D., Ranganath, R., and Blei, D. (2017a).
\newblock Hierarchical implicit models and likelihood-free variational
  inference.
\newblock In {\em Advances in Neural Information Processing Systems}, pages
  5523--5533.

\bibitem[Tran et~al., 2017b]{tran2017variational}
Tran, M.-N., Nott, D.~J., and Kohn, R. (2017b).
\newblock {Variational Bayes with intractable likelihood}.
\newblock {\em Journal of Computational and Graphical Statistics},
  26(4):873--882.

\bibitem[Wilkinson, 2014]{wilkinson2014accelerating}
Wilkinson, R. (2014).
\newblock Accelerating abc methods using gaussian processes.
\newblock In {\em Artificial Intelligence and Statistics}, pages 1015--1023.

\bibitem[Wood, 2010]{wood2010statistical}
Wood, S.~N. (2010).
\newblock Statistical inference for noisy nonlinear ecological dynamic systems.
\newblock {\em Nature}, 466(7310):1102.

\end{thebibliography}

\onecolumn
\newpage
\appendix
	
\section{Theoretical Guarantees on Convergence}

	We provide theoretical guarantees that establish convergence of the \acrfull{KELFI} framework. \Cref{sec:kernel_properties} begins by summarizing the properties of kernels used in \gls{KELFI} and introducing relevant quantities. \Cref{sec:condition_mean_embedding,,sec:empirical_condition_mean_embedding} provide an overview of \acrfullpl{CME} and their empirical estimates respectively in the context of \gls{KELFI}. \Cref{sec:general_convergence_theorems} establishes general convergence theorems for estimators based on the \gls{CME}. Using these results, we prove convergence guarantees for the \acrfull{KML}, \acrfull{MKML}, \acrfull{KMP}, and \acrfull{KMPE} in \cref{sec:kernel_means_likelihood_convergence,,sec:marginal_kernel_means_likelihood_convergence,,sec:kernel_means_posterior_convergence,,sec:kernel_means_posterior_embedding_convergence} respectively.
	
	\subsection{Kernel Properties}
	\label{sec:kernel_properties}
		
		The \gls{KELFI} framework uses a data kernel $k : \mathcal{D} \times \mathcal{D} \to \mathbb{R}$ where $\mathcal{X}, \mathcal{Y} \subseteq \mathcal{D}$. We do not assume that $\mathcal{X}$ and $\mathcal{Y}$ are necessarily the same. For example, it is possible to record observation $\bvec{y}$ in which the simulator $p(\bvec{x} | \bm{\theta})$ can never generate or fully recover, such as when $\mathcal{X} \subset \mathcal{Y}$. Conversely, it is also possible that the simulator $p(\bvec{x} | \bm{\theta})$ can generate a larger variety of simulations $\bvec{x}$ than that is possible to observe, such as when $\mathcal{Y} \subset \mathcal{X}$. It can also be neither of such cases such as when $\mathcal{X}$ an $\mathcal{Y}$ only have some overlap. However, since we assume $\mathcal{X}, \mathcal{Y} \subseteq \mathcal{D}$, the kernel $k$ is able to measure the similarity between simulated data $\bvec{x} \in \mathcal{X} \subseteq \mathcal{D}$ and observed data $\bvec{y} \in \mathcal{Y} \subseteq \mathcal{D}$.
		
		The \gls{KELFI} framework employs bounded symmetric positive definite kernels $\ell$ and $k$. Because they are bounded, we can explicitly denote the following upper bounds to their \gls{RKHS} norm,
		\begin{equation}
			\bar{\ell} := \sup_{\bm{\theta} \in \vartheta} \| \ell(\bm{\theta}, \cdot) \|_{\mathcal{H}_{\ell}} = \sup_{\bm{\theta} \in \vartheta} \sqrt{\ell(\bm{\theta}, \bm{\theta})},
		\end{equation}
		\begin{equation}
			\bar{k} := \sup_{\bvec{d} \in \mathcal{D}} \| k(\bvec{d}, \cdot) \|_{\mathcal{H}_{k}} = \sup_{\bvec{d} \in \mathcal{D}} \sqrt{k(\bvec{d}, \bvec{d})}.
		\end{equation}
		When $\ell$ and $k$ are stationary, we have $\bar{\ell} = \sqrt{\ell(\bvec{0}, \bvec{0})}$ and $\bar{k} = \sqrt{k(\bvec{0}, \bvec{0})}$. 
		
		In the \gls{KELFI} framework, we first select the $\epsilon$-kernel $\kappa_{\epsilon}$. Based on this the choice of the $\epsilon$-kernel, we then select the kernel $k$ to satisfy
		\begin{equation}
			\kappa_{\epsilon}(\bvec{y}, \bvec{x}) = c_{\epsilon} k(\bvec{y}, \bvec{x}),
		\end{equation}
		where $c_{\epsilon} > 0$ is a scaling constant to ensure that $\kappa_{\epsilon}(\bvec{y}, \bvec{x}) = p_{\epsilon}(\bvec{y} | \bvec{x})$ is a normalized density on $\mathcal{Y}$. In contrast, the kernel $k$ has no such restriction. Since it is a scaled version of $k$, $\kappa_{\epsilon}$ is also bounded symmetric positive definite as a function of $\bvec{x}$ and $\bvec{y}$. In this way, $\kappa_{\epsilon}(\bvec{d}, \cdot) \in \mathcal{H}_{k}$ is always in the \gls{RKHS} $\mathcal{H}_{k}$ characterized by $k$ for all $\bvec{d} \in \mathcal{D}$. As a consequence, $\epsilon$ is also a hyperparameter of $k$, although this is not explicitly notated for brevity.
		
		Since $\bvec{y} \in \mathcal{Y} \subseteq \mathcal{D}$, we have $\kappa_{\epsilon}(\bvec{y}, \cdot) \in \mathcal{H}_{k}$. We can then find its \gls{RKHS} norm, 
		\begin{equation}
			\| \kappa_{\epsilon}(\bvec{y}, \cdot) \|_{\mathcal{H}_{k}} = c_{\epsilon} \| k(\bvec{y}, \cdot) \|_{\mathcal{H}_{k}} = c_{\epsilon} \sqrt{k(\bvec{y}, \bvec{y})} = \sqrt{c_{\epsilon}} \sqrt{c_{\epsilon} k(\bvec{y}, \bvec{y})} = \sqrt{c_{\epsilon}} \sqrt{\kappa_{\epsilon}(\bvec{y}, \bvec{y})},
		\end{equation}
		which is different to $\| \kappa_{\epsilon}(\bvec{y}, \cdot) \|_{\mathcal{H}_{\kappa_{\epsilon}}} = \sqrt{\kappa_{\epsilon}(\bvec{y}, \bvec{y})}$. Therefore, while the \gls{KELFI} algorithm only requires $\kappa_{\epsilon}$ to be specified and $k$ is not explicitly used, this subtle difference is a reminder that $k$ is the underlying kernel that defines the \gls{RKHS}, not $\kappa_{\epsilon}$.
		As a consequence, we have that the upper bound to the \gls{RKHS} norm of $\kappa_{\epsilon}$ satisfies
		\begin{equation}
			\bar{\kappa}_{\epsilon} := \sup_{\bvec{d} \in \mathcal{D}} \| \kappa_{\epsilon}(\bvec{d}, \cdot) \|_{\mathcal{H}_{k}} = \sqrt{c_{\epsilon}} \sup_{\bvec{d} \in \mathcal{D}} \sqrt{\kappa_{\epsilon}(\bvec{d}, \bvec{d})}.
		\end{equation}
		Furthermore, if $\kappa_{\epsilon}$ is stationary, then $\kappa_{\epsilon}(\bvec{d}, \bvec{d}) = \kappa_{\epsilon}(\bvec{0}, \bvec{0})$ for all $\bvec{d} \in \mathcal{D}$. A typical example is the Gaussian density $\kappa_{\epsilon}(\bvec{y}, \bvec{x}) = \mathcal{N}(\bvec{y} | \bvec{x}, \epsilon^{2} I)$. In this case, $c_{\epsilon} = 1 / (\sqrt{2 \pi} \epsilon)^{n}$ and $\kappa_{\epsilon}(\bvec{y}, \bvec{y}) = 1 / (\sqrt{2 \pi} \epsilon)^{n}$ are the same, and thus $\| \kappa_{\epsilon}(\bvec{y}, \cdot) \|_{\mathcal{H}_{k}} = 1 / (\sqrt{2 \pi} \epsilon)^{n} = c_{\epsilon}$. The corresponding kernel $k$ is the isotropic Gaussian kernel
		
		When $\mathcal{D} = \mathbb{R}^{n}$, the most commonly used kernel for the \gls{KELFI} framework is the anisotropic Gaussian kernel where each dimension uses a potentially different length scale $\sigma_{i}$. When its length scales are learned via some hyperparameter learning algorithm, it is also referred to as the \gls{ARD} kernel. This kernel has the following form,
		\begin{equation}
			k(\bvec{x}, \bvec{x}') = \exp{\bigg( -\frac{1}{2} \sum_{i = 1}^{n} \Big( \frac{x_{i} - x_{i}'}{\sigma_{i}} \Big)^{2} \bigg)}.
		\end{equation}
		Since $\kappa_{\epsilon}(\bvec{y}, \bvec{x}) = c_{\epsilon} k(\bvec{y}, \bvec{x})$, this means that the length scales are simply the \gls{ABC} tolerance $\sigma_{i} = \epsilon_{i}$ for $i \in [n]$, and that there can be a separate tolerance for each dimension of the data or summary statistic. Similarly, when $\vartheta = \mathbb{R}^{D}$, we also often employ the \gls{ARD} kernel for $\ell$, but we use $\beta_{d}$, $d \in [D]$, to denote the length scales.

	\subsection{Conditional Mean Embedding}
	\label{sec:condition_mean_embedding}
				
		To construct a conditional mean operator $\mathcal{U}_{\bvec{X} | \bm{\Theta}}$ corresponding to the distribution $p(\bvec{x} | \bm{\theta})$, we first choose a kernel $\ell : \vartheta \times \vartheta \to \mathbb{R}$ for domain $\vartheta$ and another kernel $k : \mathcal{D} \times \mathcal{D} \to \mathbb{R}$ for domain $\mathcal{D}$. These kernels $\ell$ and $k$ each describe how similarity is measured within their respective domains, and are bounded symmetric positive definite such that they uniquely define the \gls{RKHS} $\mathcal{H}_{\ell}$ and $\mathcal{H}_{k}$.
			
		The conditional mean operator $\mathcal{U}_{\bvec{X} | \bm{\Theta}} : \mathcal{H}_{\ell} \to \mathcal{H}_{k}$ is defined by the equation $\mu_{\bvec{X} | \bm{\Theta} = \bm{\theta}} = \mathcal{U}_{\bvec{X} | \bm{\Theta}} \ell(\bm{\theta}, \cdot)$, where $\mu_{\bvec{X} | \bm{\Theta} = \bm{\theta}}$ is the \gls{CME} defined by
		\begin{equation}
			\mu_{\bvec{X} | \bm{\Theta} = \bm{\theta}} := \mathbb{E}[k(\bvec{X}, \cdot) | \bm{\Theta} = \bm{\theta}].
		\end{equation}
		In this sense, $\mathcal{U}_{\bvec{X} | \bm{\Theta}}$ sweeps out a family of \glspl{CME} $\mu_{\bvec{X} | \bm{\Theta} = \bm{\theta}} \in \mathcal{H}_{k}$, each indexed by $\bm{\theta} \in \vartheta$.
		
		We then define cross covariance operators $C_{\bvec{X} \bm{\Theta}} := \mathbb{E}[k(\bvec{X}, \cdot) \otimes \ell(\bm{\Theta}, \cdot)] : \mathcal{H}_{\ell} \to \mathcal{H}_{k}$ and $C_{\bm{\Theta} \bm{\Theta}} := \mathbb{E}[\ell(\bm{\Theta}, \cdot) \otimes \ell(\bm{\Theta}, \cdot)] : \mathcal{H}_{\ell} \to \mathcal{H}_{\ell}$. Alternatively, they can be seen as elements within the tensor product space $C_{\bvec{X} \bm{\Theta}} \in \mathcal{H}_{k} \otimes \mathcal{H}_{\ell}$ and $C_{\bm{\Theta} \bm{\Theta}} \in \mathcal{H}_{\ell} \otimes \mathcal{H}_{\ell}$. That is, they are second order mean embeddings.
		
		Under the assumption that $\ell(\bm{\theta}, \cdot) \in \mathrm{image}(C_{\bm{\Theta} \bm{\Theta}})$, it can be shown that $\mathcal{U}_{\bvec{X} | \bm{\Theta}} = C_{\bvec{X} \bm{\Theta}} (C_{\bm{\Theta} \bm{\Theta}})^{-1}$. While this assumption is satisfied for finite domains $\vartheta$ with a characteristic kernel $\ell$, it does not necessarily hold when $\vartheta$ is a continuous domain \citep{fukumizu2004dimensionality}. Instead, in this case $C_{\bvec{X} \bm{\Theta}} (C_{\bm{\Theta} \bm{\Theta}})^{-1}$ becomes only an approximation to $\mathcal{U}_{\bvec{X} | \bm{\Theta}}$, and we instead regularize the inversion with a regularization hyperparameter $\lambda \geq 0$ and use $\mathcal{U}_{\bvec{X} | \bm{\Theta}} = C_{\bvec{X} \bm{\Theta}} (C_{\bm{\Theta} \bm{\Theta}} + \lambda I)^{-1}$, which also serves to avoid overfitting \citep{song2013kernel}. This relaxation can be applied to all subsequent results and theorems.

	\subsection{Empirical Estimate for the Conditional Mean Embedding}
	\label{sec:empirical_condition_mean_embedding}
	
		Suppose $\{\bm{\theta}_{j}, \bvec{x}_{j}\} \sim p(\bvec{x} | \bm{\theta}) \pi(\bm{\theta})$ are \textit{iid} across $j \in [m]$. The conditional mean operator $\mathcal{U}_{\bvec{X} | \bm{\Theta}}$ is estimated by
		\begin{equation}
			\hat{\mathcal{U}}_{\bvec{X} | \bm{\Theta}} = \Phi (L + m \lambda I)^{-1} \Psi^{T},
		\end{equation}
		where $\Phi := \begin{bmatrix} k(\bvec{x}_{1}, \cdot) & \cdots & k(\bvec{x}_{m}, \cdot) \end{bmatrix}$, $\Psi := \begin{bmatrix} \ell(\bm{\theta}_{1}, \cdot) & \cdots & \ell(\bm{\theta}_{m}, \cdot) \end{bmatrix}$, and $L := \{\ell(\bm{\theta}_{i}, \bm{\theta}_{j})\}_{i, j = 1}^{m}$. The \gls{CME} can then be estimated by
		\begin{equation}
			\hat{\mu}_{\bvec{X} | \bm{\Theta} = \bm{\theta}} = \hat{\mathcal{U}}_{\bvec{X} | \bm{\Theta}} \ell(\bm{\theta}, \cdot) = \Phi (L + m \lambda I)^{-1} \bm{\ell}(\bm{\theta})
		\end{equation}
		where  $\bm{\ell}(\bm{\theta}) := \{\ell(\bm{\theta}_{j}, \bm{\theta})\}_{j = 1}^{m}$ \citep{song2009hilbert}.
		
		For any function $f \in \mathcal{H}_{k}$, the conditional expectation of $f$ under $p(\bvec{x} | \bm{\theta})$, or $g(\bm{\theta}) := \mathbb{E}[f(\bvec{X}) | \bm{\Theta} = \bm{\theta}]$, can be approximated by the inner product $\hat{g}(\bm{\theta}) := \langle f, \hat{\mu}_{\bvec{X} | \bm{\Theta} = \bm{\theta}} \rangle_{\mathcal{H}_{k}}$ by using an empirical \gls{CME} $\hat{\mu}_{\bvec{X} | \bm{\Theta} = \bm{\theta}}$. Letting $\bvec{f} := \{f(\bvec{x}_{j})\}_{j = 1}^{m}$, this approximation admits the following form, 
		\begin{equation}
			\hat{g}(\bm{\theta}) = \bvec{f}^{T} (L + m \lambda I)^{-1} \bm{\ell}(\bm{\theta}).
		\end{equation}
		Importantly, $\hat{\mu}_{\bvec{X} | \bm{\Theta} = \bm{\theta}}$ is estimated from \textit{joint} samples $\{\bm{\theta}_{j}, \bvec{x}_{j}\}_{j = 1}^{m}$, even though it is encoding the corresponding conditional distribution $p(\bvec{x} | \bm{\theta})$. It is this fact that allows for an arbitrary choice $\pi(\bm{\theta})$ on the marginal distribution of $\bm{\Theta}$, which does not necessarily need to be the same as $p(\bm{\theta})$.
		
		Under the assumption that $\ell(\bm{\theta}, \cdot) \in \mathrm{image}(C_{\bm{\Theta} \bm{\Theta}})$, the empirical \gls{CME} $\hat{\mu}_{\bvec{X} | \bm{\Theta} = \bm{\theta}}$ converges to the true \gls{CME} $\mu_{\bvec{X} | \bm{\Theta} = \bm{\theta}}$ in \gls{RKHS} norm at rate $O_{p}((m \lambda)^{-\frac{1}{2}} + \lambda^{\frac{1}{2}})$ \cite[Theorem 6]{song2009hilbert}. That is,
			\begin{equation}
			\begin{aligned}
				&\forall \bm{\theta} \in \vartheta, \; \forall \epsilon > 0, \; \exists M_{\epsilon} > 0 \quad s.t.  \\
				&\mathbb{P}\Big[\big\| \hat{\mu}_{\bvec{X} | \bm{\Theta} = \bm{\theta}} - \mu_{\bvec{X} | \bm{\Theta} = \bm{\theta}} \big\|_{\mathcal{H}_{k}} > M_{\epsilon} \Big((m \lambda)^{-\frac{1}{2}} + \lambda^{\frac{1}{2}}\Big)\Big] < \epsilon.
			\label{eq:empirical_conditional_mean_embedding_stochastic_convergence}
			\end{aligned}
			\end{equation}
		Consequently, the empirical \gls{CME} converges at rate $O_{p}(m^{-\frac{1}{4}})$ if $\lambda$ is chosen to decay at rate $O_{p}(m^{-\frac{1}{2}})$, and often better convergence rates can be achieved under appropriate assumptions on $p(\bvec{x} | \bm{\theta})$ \citep{song2013kernel}. Again, the regularization hyperparameter $\lambda$ relaxes the assumption that $\ell(\bm{\theta}, \cdot) \in \mathrm{image}(C_{\bm{\Theta} \bm{\Theta}})$.
		
		Finally, since $\hat{\mu}_{\bvec{X} | \bm{\Theta} = \bm{\theta}} = \hat{\mathcal{U}}_{\bvec{X} | \bm{\Theta}} \ell(\bm{\theta}, \cdot)$ convergences to $\mu_{\bvec{X} | \bm{\Theta} = \bm{\theta}} = \mathcal{U}_{\bvec{X} | \bm{\Theta}} \ell(\bm{\theta}, \cdot)$ in \gls{RKHS} norm at rate $O_{p}((m \lambda)^{-\frac{1}{2}} + \lambda^{\frac{1}{2}})$ for all $\bm{\theta} \in \vartheta$ and $\ell(\bm{\theta}, \cdot)$ does not depend on $m$, we also have that $\hat{\mathcal{U}}_{\bvec{X} | \bm{\Theta}}$ converges to $\mathcal{U}_{\bvec{X} | \bm{\Theta}}$ in \gls{HS} norm at the same rate. That is,
			\begin{equation}
			\begin{aligned}
				&\forall \epsilon > 0, \; \exists M_{\epsilon} > 0 \quad s.t.  \\
				&\mathbb{P}\Big[\big\| \hat{\mathcal{U}}_{\bvec{X} | \bm{\Theta}} - \mathcal{U}_{\bvec{X} | \bm{\Theta}} \big\|_{HS} > M_{\epsilon} \Big((m \lambda)^{-\frac{1}{2}} + \lambda^{\frac{1}{2}}\Big)\Big] < \epsilon.
			\label{eq:empirical_conditional_mean_operator_stochastic_convergence}
			\end{aligned}
			\end{equation}

	\subsection{General Convergence Theorems}
	\label{sec:general_convergence_theorems}
	
		 We now establish some general convergence theorems for estimators based on inner products with the \gls{CME}. The aim is to provide a sense of the stochastic convergence of any estimator $\hat{a}$ to its true quantity $a$ with respect to some metric $d(\hat{a}, a)$. We do this by showing that either $\| \hat{\mu}_{\bvec{X} | \bm{\Theta} = \bm{\theta}} - \mu_{\bvec{X} | \bm{\Theta} = \bm{\theta}} \|_{\mathcal{H}_{k}}$ or $\| \hat{\mathcal{U}}_{\bvec{X} | \bm{\Theta}} - \mathcal{U}_{\bvec{X} | \bm{\Theta}} \|_{HS}$ is an upper bound of $d(\hat{a}, a)$ up to a scaling constant.
		 
		\begin{lemma}
			\label{thm:pointwise_uniform_convergence}
			Suppose that $\ell(\bm{\theta}, \cdot) \in \mathrm{image}(C_{\bm{\Theta} \bm{\Theta}})$ and that there exists $0 \leq \gamma < \infty$ such that for some estimator $\hat{a}$, target $a$, and metric $d(\hat{a}, a)$,
			\begin{equation}
				d(\hat{a}, a) \leq \gamma \big\| \hat{\mathcal{U}}_{\bvec{X} | \bm{\Theta}} - \mathcal{U}_{\bvec{X} | \bm{\Theta}} \big\|_{HS},
			\label{eq:estimator_error_bound}
			\end{equation}
			then the estimator $\hat{a}$ converges to the target $a$ with respect to the metric $d$ at rate $O_{p}((m \lambda)^{-\frac{1}{2}} + \lambda^{\frac{1}{2}})$.
		\end{lemma}
	
		\begin{proof}
			Suppose that there exists $0 \leq \gamma < \infty$ such that \eqref{eq:estimator_error_bound} is satisfied. That is, the inequality \eqref{eq:estimator_error_bound} holds for all possible data observations $\{\bm{\theta}_{j}, \bvec{x}_{j}\}_{j = 1}^{m}$. For any constant $C$, the implication statement $\big\| \hat{\mathcal{U}}_{\bvec{X} | \bm{\Theta}} - \mathcal{U}_{\bvec{X} | \bm{\Theta}} \big\|_{HS} \leq C \implies d(\hat{a}, a) \leq C \gamma $ holds for all possible observation events $\omega \in \Omega$. Writing this explicitly in event space translates this to a statement of probability inequality,
			\begin{equation}
			\begin{aligned}
				\{\omega \in \Omega : \big\| \hat{\mathcal{U}}_{\bvec{X} | \bm{\Theta}} - \mathcal{U}_{\bvec{X} | \bm{\Theta}} \big\|_{HS} \leq C\} &\subseteq \{\omega \in \Omega : d(\hat{a}, a) \leq C \gamma\} \\
				\implies \mathbb{P}\Big[\big\| \hat{\mathcal{U}}_{\bvec{X} | \bm{\Theta}} - \mathcal{U}_{\bvec{X} | \bm{\Theta}} \big\|_{HS} \leq C\Big] &\leq \mathbb{P}\Big[d(\hat{a}, a) \leq C \gamma \Big].
			\label{eq:probability_statement}
			\end{aligned}
			\end{equation}
			Since we assume that $\ell(\bm{\theta}, \cdot) \in \mathrm{image}(C_{\bm{\Theta} \bm{\Theta}})$, statement \eqref{eq:empirical_conditional_mean_embedding_stochastic_convergence} is valid. By letting $C = M_{\epsilon} ((m \lambda)^{-\frac{1}{2}} + \lambda^{\frac{1}{2}})$ in \eqref{eq:probability_statement}, we immediately have that the probability inequality in statement \eqref{eq:empirical_conditional_mean_operator_stochastic_convergence} is also true if we replace $\big\| \hat{\mathcal{U}}_{\bvec{X} | \bm{\Theta}} - \mathcal{U}_{\bvec{X} | \bm{\Theta}} \big\|_{HS}$ with $d(\hat{a}, a)$ and $M_{\epsilon}$ with $\gamma M_{\epsilon}$,	
			\begin{equation}
			\begin{aligned}
				\mathbb{P}\Big[\big\| \hat{\mathcal{U}}_{\bvec{X} | \bm{\Theta}} - \mathcal{U}_{\bvec{X} | \bm{\Theta}} \big\|_{HS} > M_{\epsilon} \Big((m \lambda)^{-\frac{1}{2}} + \lambda^{\frac{1}{2}}\Big)\Big] &< \epsilon \\
				\implies 1 - \mathbb{P}\Big[\big\| \hat{\mathcal{U}}_{\bvec{X} | \bm{\Theta}} - \mathcal{U}_{\bvec{X} | \bm{\Theta}} \big\|_{HS} \leq M_{\epsilon} \Big((m \lambda)^{-\frac{1}{2}} + \lambda^{\frac{1}{2}}\Big)\Big] &< \epsilon \\
				\implies \mathbb{P}\Big[\big\| \hat{\mathcal{U}}_{\bvec{X} | \bm{\Theta}} - \mathcal{U}_{\bvec{X} | \bm{\Theta}} \big\|_{HS} \leq M_{\epsilon} \Big((m \lambda)^{-\frac{1}{2}} + \lambda^{\frac{1}{2}}\Big)\Big] &> 1 -  \epsilon \\
				\implies \mathbb{P}\Big[d(\hat{a}, a) \leq \gamma M_{\epsilon} \Big((m \lambda)^{-\frac{1}{2}} + \lambda^{\frac{1}{2}}\Big)\Big] &> 1 -  \epsilon \\
				\implies 1 - \mathbb{P}\Big[d(\hat{a}, a) \leq \gamma M_{\epsilon} \Big((m \lambda)^{-\frac{1}{2}} + \lambda^{\frac{1}{2}}\Big)\Big] &< \epsilon \\
				\implies \mathbb{P}\Big[d(\hat{a}, a) > \gamma M_{\epsilon} \Big((m \lambda)^{-\frac{1}{2}} + \lambda^{\frac{1}{2}}\Big)\Big] &< \epsilon,
			\end{aligned}	
			\end{equation}
			where we employed statement \eqref{eq:probability_statement} between the third and fourth line for $C = M_{\epsilon} ((m \lambda)^{-\frac{1}{2}} + \lambda^{\frac{1}{2}})$. Therefore, since $M_{\epsilon}$ is arbitrary, define $\tilde{M}_{\epsilon} := \gamma M_{\epsilon}$ so that the following statement holds,
			\begin{equation}
				\forall \epsilon > 0, \; \exists \tilde{M}_{\epsilon} > 0 \quad s.t. \quad \mathbb{P}\Big[d(\hat{a}, a) > \tilde{M}_{\epsilon} \Big((m \lambda)^{-\frac{1}{2}} + \lambda^{\frac{1}{2}}\Big)\Big] < \epsilon.
			\end{equation}
			In other words, the estimator $\hat{a}$ stochastically converges to $a$ at a rate of at least $O_{p}((n \lambda)^{-\frac{1}{2}} + \lambda^{\frac{1}{2}})$ with respect to the metric $d$.
		\end{proof}
	
		\begin{lemma}
			\label{thm:pointwise_uniform_convergence_cme}
			Suppose that $\ell(\bm{\theta}, \cdot) \in \mathrm{image}(C_{\bm{\Theta} \bm{\Theta}})$ and that there exists $0 \leq \gamma < \infty$ such that for some estimator $\hat{a}$, target $a$, and metric $d(\hat{a}, a)$,
			\begin{equation}
				d(\hat{a}, a) \leq \gamma \big\| \hat{\mu}_{\bvec{X} | \bm{\Theta} = \bm{\theta}} - \mu_{\bvec{X} | \bm{\Theta} = \bm{\theta}} \big\|_{\mathcal{H}_{k}},
			\label{eq:estimator_error_bound_cme}
			\end{equation}
			then the estimator $\hat{a}$ converges to the target $a$ with respect to the metric $d$ at rate $O_{p}((m \lambda)^{-\frac{1}{2}} + \lambda^{\frac{1}{2}})$.
		\end{lemma}
		
		\begin{proof}
			The proof is identical to the proof for \cref{thm:pointwise_uniform_convergence}, where $\big\| \hat{\mathcal{U}}_{\bvec{X} | \bm{\Theta}} - \mathcal{U}_{\bvec{X} | \bm{\Theta}} \big\|_{HS}$ is replaced with $\big\| \hat{\mu}_{\bvec{X} | \bm{\Theta} = \bm{\theta}} - \mu_{\bvec{X} | \bm{\Theta} = \bm{\theta}} \big\|_{\mathcal{H}_{k}}$ throughout. Alternatively, since $\big\| \hat{\mu}_{\bvec{X} | \bm{\Theta} = \bm{\theta}} - \mu_{\bvec{X} | \bm{\Theta} = \bm{\theta}} \big\|_{\mathcal{H}_{k}} = \big\| (\hat{\mathcal{U}}_{\bvec{X} | \bm{\Theta}} - \mathcal{U}_{\bvec{X} | \bm{\Theta}}) \ell(\bm{\theta}, \cdot) \big\|_{\mathcal{H}_{k}} \leq \big\| \hat{\mathcal{U}}_{\bvec{X} | \bm{\Theta}} - \mathcal{U}_{\bvec{X} | \bm{\Theta}} \big\|_{HS} \big\| \ell(\bm{\theta}, \cdot) \big\|_{\mathcal{H}_{\ell}} = \big\| \hat{\mathcal{U}}_{\bvec{X} | \bm{\Theta}} - \mathcal{U}_{\bvec{X} | \bm{\Theta}} \big\|_{HS} \sqrt{\ell(\bm{\theta}, \bm{\theta})}$, $\forall \bm{\theta} \in \vartheta$, we have $d(\hat{a}, a) \leq \gamma \ell(\bm{\theta}, \bm{\theta}) \big\| \hat{\mathcal{U}}_{\bvec{X} | \bm{\Theta}} - \mathcal{U}_{\bvec{X} | \bm{\Theta}} \big\|_{HS} \leq \gamma ( \sup_{\bm{\theta} \in \vartheta} \sqrt{\ell(\bm{\theta}, \bm{\theta})} ) \big\| \hat{\mathcal{U}}_{\bvec{X} | \bm{\Theta}} - \mathcal{U}_{\bvec{X} | \bm{\Theta}} \big\|_{HS} = \gamma \bar{\ell} \big\| \hat{\mathcal{U}}_{\bvec{X} | \bm{\Theta}} - \mathcal{U}_{\bvec{X} | \bm{\Theta}} \big\|_{HS}$, $\forall \bm{\theta} \in \vartheta$. Since $\gamma \bar{\ell}$ is finite and does not depend on $m$, we apply \cref{thm:pointwise_uniform_convergence} to arrive at \cref{thm:pointwise_uniform_convergence_cme}.
		\end{proof}
				
		With \cref{thm:pointwise_uniform_convergence,,thm:pointwise_uniform_convergence_cme}, we are now equipped to show the convergence of various estimators based on \glspl{CME}. 

	\subsection{Convergence Guarantees for Kernel Means Likelihood}
	\label{sec:kernel_means_likelihood_convergence}
		
		In all subsequent theorems and proofs, recall that the approximate surrogate densities $q$ depend on $m$ and $\epsilon$, as well as other kernel and regularization hyperparameters, even though this is not explicitly notated.
		
		\begin{theorem}
			\label{thm:kernel_means_likelihood}
			Assume $\ell(\bm{\theta}, \cdot) \in \mathrm{image}(C_{\bm{\Theta} \bm{\Theta}})$. The \acrfull{KML} $q(\bvec{y} | \bm{\theta})$ converges to the likelihood $p_{\epsilon}(\bvec{y} | \bm{\theta})$ uniformly at rate $O_{p}((m \lambda)^{-\frac{1}{2}} + \lambda^{\frac{1}{2}})$ as a function of $\bm{\theta} \in \vartheta$ and $\bvec{y} \in \mathcal{Y}$.
		\end{theorem}
		
		\begin{proof}
			Consider the absolute difference between the \gls{KML} $q(\bvec{y} | \bm{\theta})$ and the likelihood $p_{\epsilon}(\bvec{y} | \bm{\theta})$,
			\begin{equation}
			\begin{aligned}
				| q(\bvec{y} | \bm{\theta}) - p_{\epsilon}(\bvec{y} | \bm{\theta})) | =& | \langle \kappa_{\epsilon}(\bvec{y}, \cdot), \hat{\mu}_{\bvec{X} | \bm{\Theta} = \bm{\theta}} \rangle_{\mathcal{H}_{k}} - \langle \kappa_{\epsilon}(\bvec{y}, \cdot), \mu_{\bvec{X} | \bm{\Theta} = \bm{\theta}} \rangle_{\mathcal{H}_{k}}| \\
				=& | \langle \kappa_{\epsilon}(\bvec{y}, \cdot), \hat{\mu}_{\bvec{X} | \bm{\Theta} = \bm{\theta}} - \mu_{\bvec{X} | \bm{\Theta} = \bm{\theta}} \rangle_{\mathcal{H}_{k}}| \\
				\leq& \| \kappa_{\epsilon}(\bvec{y}, \cdot) \|_{\mathcal{H}_{k}} \| \hat{\mu}_{\bvec{X} | \bm{\Theta} = \bm{\theta}} - \mu_{\bvec{X} | \bm{\Theta} = \bm{\theta}} \|_{\mathcal{H}_{k}} \\
				\leq& \bar{\kappa}_{\epsilon} \| \hat{\mu}_{\bvec{X} | \bm{\Theta} = \bm{\theta}} - \mu_{\bvec{X} | \bm{\Theta} = \bm{\theta}} \|_{\mathcal{H}_{k}} \\
				=& \bar{\kappa}_{\epsilon} \| (\hat{\mathcal{U}}_{\bvec{X} | \bm{\Theta}} - \mathcal{U}_{\bvec{X} | \bm{\Theta}}) \ell(\bm{\theta}, \cdot) \|_{\mathcal{H}_{k}} \\
				\leq& \bar{\kappa}_{\epsilon} \big\| \hat{\mathcal{U}}_{\bvec{X} | \bm{\Theta}} - \mathcal{U}_{\bvec{X} | \bm{\Theta}} \big\|_{HS}  \big\| \ell(\bm{\theta}, \cdot) \big\|_{\mathcal{H}_{\ell}} \\
				=& \bar{\kappa}_{\epsilon} \sqrt{\ell(\bm{\theta}, \bm{\theta})} \big\| \hat{\mathcal{U}}_{\bvec{X} | \bm{\Theta}} - \mathcal{U}_{\bvec{X} | \bm{\Theta}} \big\|_{HS} \\
				\leq& \bar{\kappa}_{\epsilon} \bar{\ell} \big\| \hat{\mathcal{U}}_{\bvec{X} | \bm{\Theta}} - \mathcal{U}_{\bvec{X} | \bm{\Theta}} \big\|_{HS}.
			\end{aligned}
			\end{equation}
		\end{proof}
		Since $\gamma = \bar{\kappa}_{\epsilon} \bar{\ell}$ is independent of $m$, we apply \cref{thm:pointwise_uniform_convergence} to establish the convergence. Since this upper bound does not depend on $\bm{\theta} \in \vartheta$ or $\bvec{y} \in \mathcal{Y}$ and the metric is the absolute difference, this convergence is uniform as a function of both $\bm{\theta} \in \vartheta$ and $\bvec{y} \in \mathcal{Y}$.
		
		Alternatively, convergence guarantees for the \gls{KML} can be established by its connection to the form of a \acrfull{GPR}, leveraging frameworks and properties from a regression perspective. This connection is discussed briefly in \cref{sec:connection_to_other_models}.
		
	\subsection{Convergence Guarantees for Marginal Kernel Means Likelihood}
	\label{sec:marginal_kernel_means_likelihood_convergence}
	
		\begin{theorem}
			\label{thm:marginal_kernel_means_likelihood}
			Assume $\ell(\bm{\theta}, \cdot) \in \mathrm{image}(C_{\bm{\Theta} \bm{\Theta}})$. The \acrfull{MKML} $q(\bvec{y})$ converges to the marginal likelihood $p_{\epsilon}(\bvec{y})$ uniformly at rate $O_{p}((m \lambda)^{-\frac{1}{2}} + \lambda^{\frac{1}{2}})$ as a function of $\bvec{y} \in \mathcal{Y}$.
		\end{theorem}
		
		\begin{proof}
			We begin by writing the marginalization operation as an expectation over $p(\bm{\theta})$. This gives us $q(\bvec{y}) := \int_{\vartheta} q(\bvec{y} | \bm{\theta}) p(\bm{\theta}) d \bm{\theta} = \mathbb{E}[q(\bvec{y} | \bm{\Theta})]$ and $p_{\epsilon}(\bvec{y}) := \int_{\vartheta} p_{\epsilon}(\bvec{y} | \bm{\theta}) p(\bm{\theta}) d \bm{\theta} = \mathbb{E}[p_{\epsilon}(\bvec{y} | \bm{\Theta})]$. Consider the absolute difference between the  \gls{MKML} $q(\bvec{y})$ and the marginal likelihood $p_{\epsilon}(\bvec{y})$,
			\begin{equation}
			\begin{aligned}
				| q(\bvec{y}) - p_{\epsilon}(\bvec{y}) | &= | \mathbb{E}[q(\bvec{y} | \bm{\Theta}) - p(\bvec{y} | \bm{\Theta})] | \\
				&\leq \mathbb{E}[ | q(\bvec{y} | \bm{\Theta}) - p(\bvec{y} | \bm{\Theta}) | ] \\
				&\leq \bar{\kappa}_{\epsilon} \mathbb{E}[ \| \hat{\mu}_{\bvec{X} | \bm{\Theta} = \bm{\Theta}} - \mu_{\bvec{X} | \bm{\Theta} = \bm{\Theta}} \|_{\mathcal{H}_{k}} ] \\
				&= \bar{\kappa}_{\epsilon} \mathbb{E}[ \| (\hat{\mathcal{U}}_{\bvec{X} | \bm{\Theta}} - \mathcal{U}_{\bvec{X} | \bm{\Theta}}) \ell(\bm{\Theta}, \cdot) \|_{\mathcal{H}_{k}} ] \\
				&\leq \bar{\kappa}_{\epsilon} \mathbb{E}[ \| \hat{\mathcal{U}}_{\bvec{X} | \bm{\Theta}} - \mathcal{U}_{\bvec{X} | \bm{\Theta}} \|_{HS} \| \ell(\bm{\Theta}, \cdot) \|_{\mathcal{H}_{\ell}} ] \\
				&= \bar{\kappa}_{\epsilon} \mathbb{E}[ \| \hat{\mathcal{U}}_{\bvec{X} | \bm{\Theta}} - \mathcal{U}_{\bvec{X} | \bm{\Theta}} \|_{HS} \sqrt{\ell(\bm{\Theta}, \bm{\Theta})} ] \\
				&= \bar{\kappa}_{\epsilon} \mathbb{E}[ \sqrt{\ell(\bm{\Theta}, \bm{\Theta})} ]  \| \hat{\mathcal{U}}_{\bvec{X} | \bm{\Theta}} - \mathcal{U}_{\bvec{X} | \bm{\Theta}} \|_{HS} \\
				&\leq \bar{\kappa}_{\epsilon} \mathbb{E}[ \bar{\ell} ]  \| \hat{\mathcal{U}}_{\bvec{X} | \bm{\Theta}} - \mathcal{U}_{\bvec{X} | \bm{\Theta}} \|_{HS} \\
				&= \bar{\kappa}_{\epsilon} \bar{\ell}  \| \hat{\mathcal{U}}_{\bvec{X} | \bm{\Theta}} - \mathcal{U}_{\bvec{X} | \bm{\Theta}} \|_{HS}
			\end{aligned}
			\end{equation}
			Since $\gamma = \bar{\kappa}_{\epsilon} \bar{\ell}$ is independent of $m$, we apply \cref{thm:pointwise_uniform_convergence} to establish the convergence. Since this upper bound does not depend on $\bvec{y} \in \mathcal{Y}$ and the metric is the absolute difference, this convergence is uniform as a function of $\bvec{y} \in \mathcal{Y}$.
		\end{proof}
			
	\subsection{Convergence Guarantees for Kernel Means Posterior}
	\label{sec:kernel_means_posterior_convergence}
	
		\begin{theorem} 
			\label{thm:kernel_means_posterior}
			Assume $\ell(\bm{\theta}, \cdot) \in \mathrm{image}(C_{\bm{\Theta} \bm{\Theta}})$ and that there exists $\delta > 0$ such that $q(\bvec{y}) \geq \delta$ for all $m \geq M$ where $M \in \mathbb{N}_{+}$. The \acrfull{KMP} $q(\bm{\theta} | \bvec{y})$ converges pointwise to the posterior $p_{\epsilon}(\bm{\theta} | \bvec{y})$ at rate $O_{p}((m \lambda)^{-\frac{1}{2}} + \lambda^{\frac{1}{2}})$ as a function of $\bm{\theta} \in \vartheta$ and $\bvec{y} \in \mathcal{Y}$. If $\sup_{\bm{\theta} \in \vartheta} p(\bm{\theta}) < \infty$ and $\sup_{\bm{\theta} \in \vartheta} p_{\epsilon}(\bvec{y} | \bm{\theta}) < \infty$, then the convergence is uniform in $\bm{\theta} \in \vartheta$. If $\sup_{\bvec{y} \in \mathcal{Y}} p_{\epsilon}(\bm{\theta} | \bvec{y}) < \infty$, then the convergence is uniform in $\bvec{y} \in \mathcal{Y}$.
		\end{theorem}
		
		\begin{proof}
			First, consider the density ratio between the approximate and true densities for the likelihood and marginal likelihood,
			\begin{equation}
			\begin{aligned}
				\bigg| \frac{q(\bvec{y} | \bm{\theta})}{p_{\epsilon}(\bvec{y} | \bm{\theta})} - 1 \bigg| \leq \frac{1}{p_{\epsilon}(\bvec{y} | \bm{\theta})} \big| q(\bvec{y} | \bm{\theta}) - p_{\epsilon}(\bvec{y} | \bm{\theta}) \big| \leq \frac{\bar{\kappa}_{\epsilon} \bar{\ell}}{p_{\epsilon}(\bvec{y} | \bm{\theta})} \big\| \hat{\mathcal{U}}_{\bvec{X} | \bm{\Theta}} - \mathcal{U}_{\bvec{X} | \bm{\Theta}} \big\|_{HS},
			\end{aligned}
			\end{equation}
			\begin{equation}
			\begin{aligned}
				\bigg| \frac{q(\bvec{y})}{p_{\epsilon}(\bvec{y})} - 1 \bigg| \leq \frac{1}{p_{\epsilon}(\bvec{y})} \big| q(\bvec{y}) - p_{\epsilon}(\bvec{y}) \big| \leq \frac{\bar{\kappa}_{\epsilon} \bar{\ell}}{p_{\epsilon}(\bvec{y})} \big\| \hat{\mathcal{U}}_{\bvec{X} | \bm{\Theta}} - \mathcal{U}_{\bvec{X} | \bm{\Theta}} \big\|_{HS}.
			\end{aligned}
			\end{equation}
			Now, consider the absolute difference between the \gls{KMP} $q(\bm{\theta} | \bvec{y})$ and the posterior $p_{\epsilon}(\bm{\theta} | \bvec{y})$ for all $m > M$.
			\begin{equation}
			\begin{aligned}
				\Big| q(\bm{\theta} | \bvec{y}) - p_{\epsilon}(\bm{\theta} | \bvec{y}) \Big| =& \bigg| \frac{q(\bvec{y} | \bm{\theta})}{q(\bvec{y})} - \frac{p_{\epsilon}(\bvec{y} | \bm{\theta})}{p_{\epsilon}(\bvec{y})} \bigg| p(\bm{\theta}) \\
				=& \bigg| \frac{q(\bvec{y} | \bm{\theta})}{p_{\epsilon}(\bvec{y} | \bm{\theta})} - \frac{q(\bvec{y})}{p_{\epsilon}(\bvec{y})} \bigg| \frac{p_{\epsilon}(\bvec{y} | \bm{\theta})p(\bm{\theta})}{|q(\bvec{y})|} \\
				=& \bigg| \Big( \frac{q(\bvec{y} | \bm{\theta})}{p_{\epsilon}(\bvec{y} | \bm{\theta})} - 1 \Big) - \Big( \frac{q(\bvec{y})}{p_{\epsilon}(\bvec{y})} - 1 \Big) \bigg| \frac{p_{\epsilon}(\bvec{y} | \bm{\theta})p(\bm{\theta})}{|q(\bvec{y})|} \\
				\leq& \bigg( \Big| \frac{q(\bvec{y} | \bm{\theta})}{p_{\epsilon}(\bvec{y} | \bm{\theta})} - 1 \Big| + \Big|\frac{q(\bvec{y})}{p_{\epsilon}(\bvec{y})} - 1 \Big| \bigg) \frac{p_{\epsilon}(\bvec{y} | \bm{\theta})p(\bm{\theta})}{|q(\bvec{y})|} \\
				\leq& \bigg( \frac{\bar{\kappa}_{\epsilon} \bar{\ell}}{p_{\epsilon}(\bvec{y} | \bm{\theta})} \Big\| \hat{\mathcal{U}}_{\bvec{X} | \bm{\Theta}} - \mathcal{U}_{\bvec{X} | \bm{\Theta}} \Big\|_{HS} + \frac{\bar{\kappa}_{\epsilon} \bar{\ell}}{p_{\epsilon}(\bvec{y})} \Big\| \hat{\mathcal{U}}_{\bvec{X} | \bm{\Theta}} - \mathcal{U}_{\bvec{X} | \bm{\Theta}} \Big\|_{HS} \bigg) \frac{p_{\epsilon}(\bvec{y} | \bm{\theta})p(\bm{\theta})}{|q(\bvec{y})|} \\
				\leq& \bigg( \bar{\kappa}_{\epsilon} \bar{\ell} p(\bm{\theta}) \Big\| \hat{\mathcal{U}}_{\bvec{X} | \bm{\Theta}} - \mathcal{U}_{\bvec{X} | \bm{\Theta}} \Big\|_{HS} + \bar{\kappa}_{\epsilon} \bar{\ell} p_{\epsilon}(\bm{\theta} | \bvec{y}) \Big\| \hat{\mathcal{U}}_{\bvec{X} | \bm{\Theta}} - \mathcal{U}_{\bvec{X} | \bm{\Theta}} \Big\|_{HS} \bigg) \frac{1}{|q(\bvec{y})|} \\
				\leq& \bar{\kappa}_{\epsilon} \bar{\ell} \big( p(\bm{\theta}) + p_{\epsilon}(\bm{\theta} | \bvec{y}) \big) \Big\| \hat{\mathcal{U}}_{\bvec{X} | \bm{\Theta}} - \mathcal{U}_{\bvec{X} | \bm{\Theta}} \Big\|_{HS} \frac{1}{|q(\bvec{y})|} \\
				\leq& \frac{\bar{\kappa}_{\epsilon} \bar{\ell} \big( p(\bm{\theta}) + p_{\epsilon}(\bm{\theta} | \bvec{y}) \big)}{\delta} \Big\| \hat{\mathcal{U}}_{\bvec{X} | \bm{\Theta}} - \mathcal{U}_{\bvec{X} | \bm{\Theta}} \Big\|_{HS}.
			\label{eq:kmp_difference_bound}
			\end{aligned}
			\end{equation}
			Since $\gamma = \frac{\bar{\kappa}_{\epsilon} \bar{\ell}}{\delta} \big( p(\bm{\theta}) + p_{\epsilon}(\bm{\theta} | \bvec{y}) \big)$ is independent of $m$ and the upper bound holds for all $m > M$, we apply \cref{thm:pointwise_uniform_convergence} to establish the convergence. Since this upper bound does depend on $\bm{\theta} \in \vartheta$ and $\bvec{y} \in \mathcal{Y}$ and the metric is the absolute difference, this convergence is pointwise as a function of $\bm{\theta} \in \vartheta$ and $\bvec{y} \in \mathcal{Y}$.
			
			Furthermore, if $\bar{p}_{\bm{\Theta}} := \sup_{\bm{\theta} \in \vartheta} p(\bm{\theta}) < \infty$ and $\bar{p}_{\bvec{Y} | \bm{\Theta}} := \sup_{\bm{\theta} \in \vartheta} p_{\epsilon}(\bvec{y} | \bm{\theta}) < \infty$, then
			\begin{equation}
			\begin{aligned}
				p(\bm{\theta}) + p_{\epsilon}(\bm{\theta} | \bvec{y}) &\leq \sup_{\bm{\theta} \in \vartheta} \big( p(\bm{\theta}) + p_{\epsilon}(\bm{\theta} | \bvec{y}) \big) \leq \sup_{\bm{\theta} \in \vartheta} p(\bm{\theta}) + \sup_{\bm{\theta} \in \vartheta} p_{\epsilon}(\bm{\theta} | \bvec{y}) \\
				&\leq \sup_{\bm{\theta} \in \vartheta} p(\bm{\theta}) + \frac{\sup_{\bm{\theta} \in \vartheta} p_{\epsilon}(\bvec{y} | \bm{\theta}) \sup_{\bm{\theta} \in \vartheta} p(\bm{\theta})}{p_{\epsilon}(\bvec{y})} \\
				&= \bar{p}_{\bm{\Theta}} + \frac{\bar{p}_{\bvec{Y} | \bm{\Theta}} \bar{p}_{\bm{\Theta}}}{p_{\epsilon}(\bvec{y})}.
			\end{aligned}
			\end{equation}
			So, $\Big| q(\bm{\theta} | \bvec{y}) - p_{\epsilon}(\bm{\theta} | \bvec{y}) \Big| \leq \frac{\bar{\kappa}_{\epsilon} \bar{\ell}}{\delta} \big( \bar{p}_{\bm{\Theta}} + \frac{\bar{p}_{\bvec{Y} | \bm{\Theta}} \bar{p}_{\bm{\Theta}}}{p_{\epsilon}(\bvec{y})} \big) \Big\| \hat{\mathcal{U}}_{\bvec{X} | \bm{\Theta}} - \mathcal{U}_{\bvec{X} | \bm{\Theta}} \Big\|_{HS}$. Since the upper bound does not depend on $\bm{\theta} \in \vartheta$, the convergence is uniform as a function of in $\bm{\theta} \in \vartheta$.
			
			Similarly, if $\bar{p}_{\bm{\Theta} | \bvec{Y}} := \sup_{\bvec{y} \in \mathcal{Y}} p_{\epsilon}(\bm{\theta} | \bvec{y}) < \infty$, then $\Big| q(\bm{\theta} | \bvec{y}) - p_{\epsilon}(\bm{\theta} | \bvec{y}) \Big| \leq \frac{\bar{\kappa}_{\epsilon} \bar{\ell}}{\delta} \big( p(\bm{\theta}) + \bar{p}_{\bm{\Theta} | \bvec{Y}} \big) \Big\| \hat{\mathcal{U}}_{\bvec{X} | \bm{\Theta}} - \mathcal{U}_{\bvec{X} | \bm{\Theta}} \Big\|_{HS}$. Since the upper bound does not depend on $\bvec{y} \in \mathcal{Y}$, the convergence is uniform as a function of in $\bvec{y} \in \mathcal{Y}$.
		\end{proof}
		
		
	\subsection{Convergence Guarantees for Kernel Means Posterior Embedding}
	\label{sec:kernel_means_posterior_embedding_convergence}
	
		\begin{theorem}
			\label{thm:kernel_means_posterior_embedding}
			Assume $\ell(\bm{\theta}, \cdot) \in \mathrm{image}(C_{\bm{\Theta} \bm{\Theta}})$ and that there exists $\delta > 0$ such that $q(\bvec{y}) \geq \delta$ for all $m \geq M$ where $M \in \mathbb{N}_{+}$. The \acrfull{KMPE} $\tilde{\mu}_{\bm{\Theta} | \bvec{Y} = \bvec{y}}$ converges in \gls{RKHS} norm to the posterior mean embedding $\mu_{\bm{\Theta} | \bvec{Y} = \bvec{y}}$ at rate $O_{p}((m \lambda)^{-\frac{1}{2}} + \lambda^{\frac{1}{2}})$.
		\end{theorem}
		
		\begin{proof}
			Since $\ell$ is a bounded kernel, let $\bar{\bar{\ell}} := \sup_{\bm{\theta} \in \vartheta} \sup_{\bm{\theta}' \in \vartheta} \ell(\bm{\theta}, \bm{\theta}') > 0$. Note that this is not necessarily the same as $\bar{\ell} := \sup_{\bm{\theta} \in \vartheta} \ell(\bm{\theta}, \bm{\theta})$. Consider the \gls{RKHS} norm of the difference between \gls{KMPE} $\tilde{\mu}_{\bm{\Theta} | \bvec{Y} = \bvec{y}}$ and the posterior mean embedding $\mu_{\bm{\Theta} | \bvec{Y} = \bvec{y}}$ for all $m > M$,
			\begin{equation}
			\begin{aligned}
				& \; \bigg\| \tilde{\mu}_{\bm{\Theta} | \bvec{Y} = \bvec{y}} - \mu_{\bm{\Theta} | \bvec{Y} = \bvec{y}} \bigg\|_{\mathcal{H}_{\ell}}^{2} \\
				=& \; \bigg\| \int_{\vartheta} \ell(\bm{\theta}, \cdot) q(\bm{\theta} | \bvec{y}) d\bm{\theta} - \int_{\vartheta} \ell(\bm{\theta}, \cdot) p_{\epsilon}(\bm{\theta} | \bvec{y}) d\bm{\theta} \bigg\|_{\mathcal{H}_{\ell}}^{2} \\
				=& \; \bigg\| \int_{\vartheta} \ell(\bm{\theta}, \cdot) \Big( q(\bm{\theta} | \bvec{y}) - p_{\epsilon}(\bm{\theta} | \bvec{y}) \Big) d\bm{\theta} \bigg\|_{\mathcal{H}_{\ell}}^{2} \\
				=& \; \bigg\langle \int_{\vartheta} \ell(\bm{\theta}, \cdot) \Big( q(\bm{\theta} | \bvec{y}) - p_{\epsilon}(\bm{\theta} | \bvec{y}) \Big) d\bm{\theta}, \int_{\vartheta} \ell(\bm{\theta}', \cdot) \Big( q(\bm{\theta}' | \bvec{y}) - p_{\epsilon}(\bm{\theta}' | \bvec{y}) \Big) d\bm{\theta}' \bigg\rangle_{\mathcal{H}_{\ell}} \\
				=& \int_{\vartheta} \int_{\vartheta} \langle \ell(\bm{\theta}, \cdot), \ell(\bm{\theta}', \cdot) \rangle_{\mathcal{H}_{\ell}} \Big( q(\bm{\theta} | \bvec{y}) - p_{\epsilon}(\bm{\theta} | \bvec{y}) \Big) \Big( q(\bm{\theta}' | \bvec{y}) - p_{\epsilon}(\bm{\theta}' | \bvec{y}) \Big) d\bm{\theta} d\bm{\theta}' \\
				=& \int_{\vartheta} \int_{\vartheta} \ell(\bm{\theta}, \bm{\theta}') \Big( q(\bm{\theta} | \bvec{y}) - p_{\epsilon}(\bm{\theta} | \bvec{y}) \Big) \Big( q(\bm{\theta}' | \bvec{y}) - p_{\epsilon}(\bm{\theta}' | \bvec{y}) \Big) d\bm{\theta} d\bm{\theta}' \\
				=& \; \bigg| \int_{\vartheta} \int_{\vartheta} \ell(\bm{\theta}, \bm{\theta}') \Big( q(\bm{\theta} | \bvec{y}) - p_{\epsilon}(\bm{\theta} | \bvec{y}) \Big) \Big( q(\bm{\theta}' | \bvec{y}) - p_{\epsilon}(\bm{\theta}' | \bvec{y}) \Big) d\bm{\theta} d\bm{\theta}' \bigg| \\
				\leq& \int_{\vartheta} \int_{\vartheta} \Big| \ell(\bm{\theta}, \bm{\theta}') \Big| \Big| q(\bm{\theta} | \bvec{y}) - p_{\epsilon}(\bm{\theta} | \bvec{y}) \Big| \Big| q(\bm{\theta}' | \bvec{y}) - p_{\epsilon}(\bm{\theta}' | \bvec{y}) \Big| d\bm{\theta} d\bm{\theta}' \\
				\leq& \int_{\vartheta} \int_{\vartheta} \bar{\bar{\ell}}^{2} \Big| q(\bm{\theta} | \bvec{y}) - p_{\epsilon}(\bm{\theta} | \bvec{y}) \Big| \Big| q(\bm{\theta}' | \bvec{y}) - p_{\epsilon}(\bm{\theta}' | \bvec{y}) \Big| d\bm{\theta} d\bm{\theta}' \\
				=& \; \bar{\bar{\ell}}^{2} \int_{\vartheta} \Big| q(\bm{\theta} | \bvec{y}) - p_{\epsilon}(\bm{\theta} | \bvec{y}) \Big| d\bm{\theta} \int_{\vartheta} \Big| q(\bm{\theta}' | \bvec{y}) - p_{\epsilon}(\bm{\theta}' | \bvec{y}) \Big| d\bm{\theta}' \\
				=& \; \bar{\bar{\ell}}^{2} \bigg( \int_{\vartheta} \Big| q(\bm{\theta} | \bvec{y}) - p_{\epsilon}(\bm{\theta} | \bvec{y}) \Big| d\bm{\theta} \bigg)^{2}.
			\end{aligned}
			\end{equation}
			We now employ inequality \eqref{eq:kmp_difference_bound} that was derived within the proof of \cref{thm:kernel_means_posterior},
			\begin{equation}
			\begin{aligned}
				\Big\| \tilde{\mu}_{\bm{\Theta} | \bvec{Y} = \bvec{y}} - \mu_{\bm{\Theta} | \bvec{Y} = \bvec{y}} \Big\|_{\mathcal{H}_{\ell}} &\leq \bar{\bar{\ell}} \int_{\vartheta} \Big| q(\bm{\theta} | \bvec{y}) - p_{\epsilon}(\bm{\theta} | \bvec{y}) \Big| d\bm{\theta} \\
				&\leq \bar{\bar{\ell}} \int_{\vartheta} \frac{\bar{\kappa}_{\epsilon} \bar{\ell} \big( p(\bm{\theta}) + p_{\epsilon}(\bm{\theta} | \bvec{y}) \big)}{\delta} \Big\| \hat{\mathcal{U}}_{\bvec{X} | \bm{\Theta}} - \mathcal{U}_{\bvec{X} | \bm{\Theta}} \Big\|_{HS} d\bm{\theta} \\
				&= \bar{\bar{\ell}} \bigg( \int_{\vartheta} \big( p(\bm{\theta}) + p_{\epsilon}(\bm{\theta} | \bvec{y}) \big) d\bm{\theta} \bigg) \frac{\bar{\kappa}_{\epsilon} \bar{\ell}}{\delta} \Big\| \hat{\mathcal{U}}_{\bvec{X} | \bm{\Theta}} - \mathcal{U}_{\bvec{X} | \bm{\Theta}} \Big\|_{HS} \\
				&= \frac{2 \bar{\kappa}_{\epsilon} \bar{\ell} \bar{\bar{\ell}}}{\delta} \Big\| \hat{\mathcal{U}}_{\bvec{X} | \bm{\Theta}} - \mathcal{U}_{\bvec{X} | \bm{\Theta}} \Big\|_{HS}.
			\end{aligned}
			\end{equation}
			Since $\gamma = \frac{2 \bar{\kappa}_{\epsilon} \bar{\ell} \bar{\bar{\ell}}}{\delta}$ is independent of $m$ and the upper bound holds for all $m > M$, we apply \cref{thm:pointwise_uniform_convergence} to establish the convergence under the \gls{RKHS} norm.
		\end{proof}









\newpage
\section{Surrogate Densities}
\label{sec:pseudo_densities}

	Instead of modeling the posterior mean embedding directly in a fashion similar to \gls{K-ABC}, \gls{KR-ABC}, and \gls{KBR}, our approach begins by using \glspl{CME} to approximate the full likelihood \eqref{eq:likelihood} first as a surrogate likelihood, the \gls{KML}. While the \gls{KML} provides an asymptotically correct surrogate for the likelihood, for finitely many simulations the \gls{KML} is not necessarily positive nor normalized. To make the \gls{KML} compatible with \gls{MCMC}-based or variational approaches would require further amendments to the \gls{KML}, ranging from simple clipping $[q(\bvec{y} | \bm{\theta})]^{+}$ or a positivity constraint in the empirical least-squares problem for the \gls{CME} weights, since \glspl{CME} can be seen as the solution to a vector valued regression problem in the \gls{RKHS} \citep{grunewalder2012conditional}. These amendments would however introduce further bias to the already biased likelihood approximation. While these biases vanishes asymptotically as the \gls{KML} approaches a valid density due to \cref{thm:kernel_means_likelihood_copy}, the asymptotic behavior is rarely reached under limited simulations, which is the scenario of interest. Instead, \gls{KELFI} performs inference by considering the surrogate posterior and its mean embedding defined directly from the \gls{KML}.
	
	Constructed from the \gls{KML}, the \gls{KMP} is also a surrogate density, although it is normalized. While the \gls{KMP} is useful for finding \gls{MAP} solutions and visualizing posterior uncertainties, we cannot directly sample from a surrogate density that is possibly non-positive. To address this, \gls{KELFI} is motivated by super-sampling of general \glspl{CME} with kernel herding \citep{chen2010super}. Although mean embeddings are strictly positive for strictly positive kernels, when they are estimated from empirical \glspl{CME}, the resulting mean embedding may not be strictly positive \citep{song2009hilbert}. Nevertheless, kernel herding can still obtain super-samples from \gls{CME} estimates which effectively minimizes the \gls{MMD} discrepancy between the original \gls{CME} estimate and the new embedding formed from super-samples. This idea has been used to sample from conditional distributions through its empirical \gls{CME} representation in \gls{KMCF} \citep{kanagawa2016filtering} and \gls{KR-ABC} \citep{kajihara2018kernel}. Furthermore, super-samples are more informative than random samples, in the sense that empirical expectations under super-samples can potentially converge faster at $O(S^{-1})$ for $S$ samples instead of $O(S^{-\frac{1}{2}})$ for random samples.
	
	In general, surrogate densities can be seen as the ``density'' of a signed measure. Most of the properties of \glspl{KME}, including injectivity between mean embeddings and distributions, remain valid for signed measures. By defining an analogous form of mean embeddings for surrogate densities, \gls{KELFI} arrives at a novel posterior mean embedding that is associated with a marginal surrogate likelihood for hyperparameter learning. 
	
	In all experiments we found that we did not need to clip the \gls{KML} or \gls{KMP} even though they are not guaranteed a-priori to be strictly positive. This is because we used an universal kernel such as a Gaussian kernel on both $\vartheta$ and $\mathcal{D}$ so that their \gls{RKHS} is dense in their respective $L^{2}$ spaces  \citep{carmeli2010vector}. Because densities and likelihoods are often square-integrable, accurate estimations can be achieved. Finally, since we use kernel herding to super-sample the \gls{KMPE}, the \gls{KMPE} is not required to be positive to begin with.
	
\section{Connections and Future Work}
\label{sec:connection_to_other_models}

	The \gls{KML} enables approximate likelihood queries at any $\bm{\theta} \in \vartheta$, even if simulation data is not available at the corresponding $\bm{\theta}$. By using the \gls{KML} as a surrogate model for the true likelihood and accepting some modeling bias, we avoid requiring multiple expensive simulations at each query $\bm{\theta}$ that is used by many \gls{MCMC}-based \gls{ABC} approaches. In fact, as a function of $\bm{\theta}$ the \gls{KML} $q(\bvec{y} | \cdot)$ is the predictive mean of a \gls{GPR} \citep{rasmussen2006gaussian} trained on observations $\{\bm{\theta}_{j}, \kappa_{\epsilon}(\bvec{y} , \bvec{x}_{j})\}_{j = 1}^{m}$ with a \gls{GP} prior $\mathcal{GP}(0, \ell)$ and Gaussian likelihood $\mathcal{N}(\bvec{0}, m \lambda I)$, since they admit the same resulting form. This connection could provide uncertainty estimates in the \gls{KML} approximation of the likelihood via the \gls{GP} predictive variance. It is possible to then use \gls{BO} \citep{snoek2012practical} or active learning methods to guide the proposal prior $\pi$ in a sequential learning fashion that will result in the more accurate \gls{KML} approximations for a fixed number $m$ of simulations.
	
	While our posterior mean embedding \eqref{eq:kernel_means_posterior_embedding} is closed-form and thus exact for the surrogate density $q(\bm{\theta} | \bvec{y})$, it is an approximation to the mean embedding $\mu_{\bm{\Theta} | \bvec{Y} = \bvec{y}} := \int_{\vartheta} \ell(\bm{\theta}, \cdot) p_{\epsilon}(\bm{\theta} | \bvec{y}) d\bm{\theta}$ of the true soft posterior $p_{\epsilon}(\bm{\theta} | \bvec{y}) \equiv p_{\bm{\Theta} | \bvec{Y}}^{(\epsilon)}(\bm{\theta} | \bvec{y})$, and converges in \gls{RKHS} norm at the same rate as the \gls{KML}. This is different in a subtle way to the \gls{CME} of the posterior used by \gls{K-ABC}, \gls{KR-ABC}, and \gls{KBR}, which in fact is an approximation to $\mu_{\bm{\Theta} | \bvec{X} = \bvec{y}} := \int_{\vartheta} \ell(\bm{\theta}, \cdot) p_{\bm{\Theta} | \bvec{X}}(\bm{\theta} | \bvec{y}) d\bm{\theta}$, the mean embedding of $p_{\bm{\Theta} | \bvec{X}}(\bm{\theta} | \bvec{y})$, which avoids using the $\epsilon$-kernel. A key difference is that there is no known associated marginal likelihood or approximations thereof for the direct posterior mean embedding, so cross validation is required for selecting the remaining kernel hyperparameters in \gls{K-ABC}, \gls{KR-ABC}, and \gls{KBR}. \gls{K-ABC} also do not address sampling, although kernel herding can be readily applied in the same way. Kernel herding is applied to \gls{KBR} in \gls{KMCF} \citep{kanagawa2016filtering} for resampling distributions represented as a \gls{CME}. We believe it would be an interesting direction to investigate the relationships between the original empirical posterior mean embedding and the surrogate posterior mean embedding.
	
	With regards to hyperparameter learning, in the \gls{KME} literature, Bayesian learning of hyperparameters in marginal mean embeddings have been addressed through a different marginal likelihood approach by placing a \gls{GP} prior on the embedding \citep{flaxman2016bayesian}. However, a general approach for learning \gls{CME} hyperparameters in a Bayesian framework remains an open question. Our simple surrogate density approach can be an alternative solution to the \gls{CME} Bayesian hyperparameter learning problem, and may lead to interesting connections.
	
	With regards to sampling, by super-sampling the surrogate posterior mean embedding, the number of posterior samples is decoupled from the number of simulations. This is unlike likelihood-free \gls{MCMC} methods for which the algorithm guides the simulator queries at parameter values that is not necessarily drawn from the prior, but rather from proposals of a Markov chain. This avoids the problem of slow mixing that is inherent in \gls{MCMC} methods, and make \gls{KELFI} more suitable for multi-modal posteriors, which remains to be experimented upon.
	
\section{Experimental Details for Blowfly}
\label{sec:blowfly}

	Our experimental setup follows that of \cite{wood2010statistical}. We adopted the 10 summary statistics used in \cite{meeds2014gps}, \cite{moreno2016automatic}, and \cite{park2016k2}, which are the log of the mean of each quartile of $\{N_{t} / 1000 \}_{t = 1}^{T}$ (4 statistics), the mean of each quartile of first-order differences of $\{N_{t} / 1000 \}_{t = 1}^{T}$ (4 statistics), and the maximal peaks of smoothed $\{N_{t}\}_{t = 1}^{T}$ with two different thresholds (2 statistics). We also use a diagonal Gaussian prior on $\log{\bm{\theta}}$ with means $[2, -1.5, 6, -1, -1, \log{(15)}]$ and standard deviations $[2, 0.5, 0.5, 1, 1, \log(5)]$. Notice that we have slightly modified the standard deviation to be broader to make the problem more challenging.

	We describe the \gls{NMSE} metric that is used to compare algorithms in our experiments. Before the experiments, we first obtain $10000$ parameter samples from the prior and simulate summary statistics from each of them. We then calculate the \glspl{MSE} of each simulated summary statistics against the observed summary statistic, and average them cross the $10000$ samples. This is now a vector of $10$ numbers, since we have an average \gls{MSE} value for each summary statistic. Those are now the \glspl{MSE} achieved under the prior. We chose $10000$ parameter samples because at this point the \glspl{MSE} for the prior has stabilized without much variance.
	
	During each experiment, we compute the \glspl{MSE} by averaging \glspl{MSE} scores across $1000$ simulations under the posterior mean or mode obtained from the algorithm. This also produces a vector of $10$ numbers. We then divide the \gls{MSE} of each statistic from the posterior by that from the prior computed earlier. This results in a vector of $10$ numbers which is now the \gls{NMSE} for the $10$ summary statistics. Since now all $10$ numbers are normalized errors with respect to the prior, we average these \gls{NMSE} scores across the statistics for a final single \gls{NMSE} score.
	
	In this way, each statistic is normalized in the final average and a \gls{NMSE} of $100\%$ correspond to the performance of the prior. Hence, the \gls{NMSE} measures the error as a percentage of the error achieved by the prior.
	
	Note that this is the \gls{NMSE} score for a particular experiment. For each algorithm, we further repeat the experiment and thus this calculation process $10$ times and show the average and the deviations in \cref{fig:inference}.
	
	For all algorithms except \gls{KBR}, we evaluate their performance by simulating from their posterior mean. For \gls{KBR} only, we simulate from its posterior mode. This is because we noticed that \gls{KBR} posterior mode decoding consistently outperformed \gls{KBR} posterior mean for the Blowfly problem. Using the posterior mode will present \gls{KBR} in its best light.
	
	We now detail the hyperparameter choices for each algorithm other than \gls{KELFI}, since most algorithms do not have a hyperparameter learning algorithm for the inference problem. Refer to their respective papers for a description of the meaning of each hyperparameter. For algorithms that use a \gls{MCMC} proposal distribution, we choose a Gaussian proposal distribution with a proposal standard deviations that are $10\%$ of the prior standard deviations. For \gls{MCMC}-\gls{ABC}, we used $\epsilon = 5$. For \gls{SL-ABC}, we used $\epsilon = 0.5$ and $S = 10$. For \gls{ASL-ABC}, we used   $S_{0} = 10$, $\epsilon = 0.5$, $\xi = 0.3$, $m = 10$, and $\Delta S = 10$. For \gls{GPS-ABC}, we used $S_{0} = 20$ samples from \gls{ASL-ABC} to initialize the \gls{GP} surrogate, and choose $\epsilon = 2$, $\xi = 0.05$, $m = 10$, and $\Delta S = 5$. For \gls{K-ABC} and \gls{KBR}, we used median length heuristic to set length scale hyperparameters, and choose $\lambda = 10^{-4}$. Note that \gls{KBR} uses two kernels on both the parameter and the summary statistics and have two regularization hyperparameters.
	
\section{Experimental Details for Lotka-Volterra}
\label{sec:lotka_volterra}

	For the Lotka-Volttera problem, our setup follows exactly as described in \cite{papamakarios2016fast}. We simulate data using the ground truth parameters and treat this as the observational data, and use it across all experiments and algorithms.
	
	In particular, the problem places a uniform prior over $\log{\bm{\theta}}$. Since the parameters are independent from each other in the prior, transforming the \gls{ABC} task into one with a Gaussian prior is straight forward by doing it separately for each parameter. To convert from $\log{\bm{\theta}}$ to $\bvec{z}$, denoting a realization of a Gaussian random variable, we first offset and scale it to a uniform in $[0, 1]$ then apply the standard normal quantile function. To convert it back, which is required before we pass our parameter query to the simulator or to present our results, we apply the standard normal cumulative distribution function and scale and offset the uniform back to its original ranges. Similar to the other experiments, we do not learn the prior hyperparameters in this paper to enable benchmarking against other methods with the same prior, so the transformed prior stay as a standard normal.

	To apply the closed-form solutions for \gls{KELFI}, we transform the prior samples into a standard Gaussian distributed samples, apply \gls{KELFI}, and transform the posterior samples back to the original space for $\log{\bm{\theta}}$.
	
	With a uniform prior and a complex intractable likelihood, the posterior is unlikely to be a Gaussian. \gls{KELFI} does not assume that the posterior is a Gaussian and thus can provide more flexible and accurate posteriors. After learning appropriate hyperparameters for \gls{KELFI} under \gls{MKML}, we draw $10000$ super-samples from the \gls{KMPE} to compute the posterior mean, and maximize the \gls{KMP} to compute the posterior mode. Finally, to compute the $95\%$ credible interval, we compute the empirical $2.5\%$ and $97.5\%$ quantile using the $10000$ super-samples.

\section{Gaussian Prior Transformations for Likelihood-Free Inference Problems}
\label{sec:gaussian_prior}

	Under certain non-exhaustive conditions, we can always transform a particular \gls{LFI} problem into another \gls{LFI} problem that involves a Gaussian prior without loss of generality. These assumptions are that $p_{\bvec{\Theta}}(\bm{\theta}) = \prod_{d = 1}^{D} p_{\Theta_{d}}(\theta_{d})$ is a continuous \gls{PDF} whose entries are independent, and that its inverse marginal \glspl{CDF} $P_{\Theta_{d}}^{-1}$ exists and is tractable.
			
	
	In terms of notation, we denote the parameters as $\bm{\theta} = \{\theta_{d}\}_{d = 1}^{D} \in \vartheta$ for $D$ parameters. For this section only, multiple \textit{iid} copies will be indexed by a superscript $\bm{\theta}^{(j)}$ for $j \in [m]$. Hence, the $d$-th parameter of the $j$-th parameter values is $\theta_{d}^{(j)}$. For densities, we use the corresponding random variable as the subscript to denote which distribution we are referring to. For example, we used $p(\bm{\theta})$ as the shorthand for the more formal notation of $p_{\bvec{\Theta}}(\bm{\theta})$ in the rest of the paper, but here we will keep the subscript to make this explicit.
			
	Suppose the original prior $p_{\bvec{\Theta}}(\bm{\theta})$ is not necessarily Gaussian, but satisfies the aforementioned assumptions. Let $\bvec{Z}$ be a random variable of the same dimensionality as $\bvec{\Theta}$ with realization $\bvec{z} \in \mathcal{Z}$. Let $p_{\bvec{Z}}(\bvec{z}) = \prod_{d = 1}^{D} p_{Z_{d}}(z_{d})$, where $p_{Z_{d}}(z_{d}) = \mathcal{N}(\mu_{d}, \sigma_{d}^{2})$ so that its density is a multivariate anisotropic Gaussian. Convenient choices that simplify transformations are $\mu_{d} = 0$ and $\sigma_{d} = \sigma$ for all $d \in [D]$, although the general methodology remains.
	
	Below we outline the general procedure for transforming a \gls{LFI} problem into another \gls{LFI} problem that involves a Gaussian prior.
			
			\begin{enumerate}
				\item Generate Gaussian samples $\bvec{z}^{(j)} \sim p_{\bvec{Z}}(\bvec{z})$ for $j \in [m]$.
				\item Convert Gaussian samples $\bvec{z}$ into uniform samples $\bvec{u}$ through $u_{d}^{(j)} = P_{Z_{d}}(z_{d}^{(j)})$ for $j \in [m]$ and $d \in [D]$. 
					\subitem That is, $\bvec{u}^{(j)} \sim U(0, 1)^{D}$ for $j \in [m]$.
				\item Convert uniform samples $\bvec{u}$ into prior samples through $\theta_{d}^{(j)} = P_{\Theta_{d}}^{-1}(u_{d}^{(j)})$ for $j \in [m]$ and $d \in [D]$.
					\subitem The overall forward transformation is $\bvec{T}(\bvec{z}) := \{T_{d}(z_{d})\}_{d = 1}^{D}$ where $T_{d}(z_{d}) = P_{\Theta_{d}}^{-1}(P_{Z_{d}}(z_{d}))$. 
					\subitem Since $P_{Z_{d}}^{-1}$ exists, the inverse transformation is $\bvec{T}^{-1}(\bm{\theta}) = \{T_{d}^{-1}(\theta_{d})\}_{d = 1}^{D}$ where $T_{d}^{-1}(\theta_{d}) = P_{Z_{d}}^{-1}(P_{\Theta_{d}}(\theta_{d}))$.
					\subitem Hence, we have $\bm{\theta}^{(j)} = \bvec{T}(\bvec{z}^{(j)})$ for $j \in [m]$.
				\item Run the simulator at the parameter samples $\bvec{x}^{(j)} \sim p_{\bvec{X} | \bm{\Theta}}(\cdot | \bm{\theta}^{(j)}) = p_{\bvec{X} | \bm{\Theta}}(\cdot | T(\bvec{z}^{(j)})) = p_{\bvec{X} | \bvec{Z}}(\cdot | \bvec{z}^{(j)})$. We now have joint samples $\{\bvec{z}^{(j)}, \bvec{x}^{(j)}\}_{j = 1}^{m}$.
				\item Use the \gls{KELFI} framework to approximate the posterior $p_{\bvec{Z} | \bvec{Y}}(\bvec{z} | \bvec{y})$ using the simulation pairs $\{\bvec{z}^{(j)}, \bvec{x}^{(j)}\}_{j = 1}^{m}$. Either we obtain the \gls{KMP} $q_{\bvec{Z} | \bvec{Y}}(\bvec{z} | \bvec{y})$, or we obtain \gls{KMPE} super-samples $\{\hat{\bvec{z}}_{s}\}_{s = 1}^{S}$.
				\item If we have samples $\{\hat{\bvec{z}}_{s}\}_{s = 1}^{S}$, then to obtain the corresponding samples for $q_{\bvec{\Theta} | \bvec{Y}}(\bm{\theta} | \bvec{y})$, we simply pass the samples $\{\hat{\bvec{z}}_{s}\}_{s = 1}^{S}$ through the transformation $\bvec{T}$ so that $\hat{\bm{\theta}}_{s} = \bvec{T}(\hat{\bvec{z}}_{s})$ for $s \in [S]$.
				\item If we have the \gls{KMP}, then to obtain the corresponding posterior density we use the standard change of variable transformation $q_{\bvec{\Theta} | \bvec{Y}}(\bm{\theta} | \bvec{y}) = q_{\bvec{Z} | \bvec{Y}}(\bvec{T}^{-1}(\bm{\theta}) | \bvec{y}) | \det{J_{\bvec{T}^{-1}}(\bm{\theta})} |$.
					\subitem The Jacobian of $\bvec{T}^{-1}$ is a $D \times D$ matrix whose $(i, j)$-th entry is $( J_{\bvec{T}^{-1}}(\bm{\theta}) )_{ij} := \frac{\partial T_{i}^{-1}}{\partial \theta_{j}}(\bm{\theta})$.
					\subitem Since the transformations of each parameter is done independently from each other, $T_{i}^{-1}$ does not depend on $\theta_{j}$ if $i \neq j$. Consequently, the Jacobian is diagonal.
					\subitem The diagonal entries are $\frac{\partial T_{i}^{-1}}{\partial \theta_{i}}(\theta_{i}) = \frac{\partial}{\partial \theta_{i}} P_{Z_{i}}^{-1}(P_{\Theta_{i}}(\theta_{i})) = (P_{Z_{i}}^{-1})'(P_{\Theta_{i}}(\theta_{i})) p_{\Theta_{i}}(\theta_{i}) = [p_{Z_{i}}(P^{-1}_{Z_{i}}(P_{\Theta_{i}}(\theta_{i})))]^{-1} p_{\Theta_{i}}(\theta_{i}) = [p_{Z_{i}}( T_{i}^{-1}(\theta_{i}) )]^{-1} p_{\Theta_{i}}(\theta_{i})$. In the second last equality we made use of the fact that the computation of the derivative of the quantile function requires only the knowledge of the density and the quantile function itself, since $(P^{-1})'(u) = (P'(P^{-1}(u)))^{-1}$. Thus, the determinant of the Jacobian is $\det{J_{\bvec{T}^{-1}}(\bm{\theta})} = \prod_{i = 1}^{d} [p_{Z_{i}}( T_{i}^{-1}(\theta_{i}) )]^{-1} p_{\Theta_{i}}(\theta_{i}) = p_{\bvec{\Theta}}(\bm{\theta}) \big[\prod_{i = 1}^{d} p_{Z_{i}}( T_{i}^{-1}(\theta_{i}) )\big]^{-1} = p_{\bvec{\Theta}}(\bm{\theta}) \big[p_{\bvec{Z}}(\bvec{T}^{-1}(\bm{\theta}))]^{-1}$.
					\subitem The change of variable transformation becomes
					\begin{equation}
						q_{\bvec{\Theta} | \bvec{Y}}(\bm{\theta} | \bvec{y}) = q_{\bvec{Z} | \bvec{Y}}(\bvec{T}^{-1}(\bm{\theta}) | \bvec{y}) \frac{p_{\bvec{\Theta}}(\bm{\theta})}{p_{\bvec{Z}}(\bvec{T}^{-1}(\bm{\theta}))}.
					\end{equation}
					\subitem Finally, the form simplifies when the form of the \gls{KMP} $q_{\bvec{Z} | \bvec{Y}}(\bvec{T}^{-1}(\bm{\theta}) | \bvec{y})$ is substituted back in,
					\begin{equation}
						q_{\bvec{\Theta} | \bvec{Y}}(\bm{\theta} | \bvec{y}) = \frac{q_{\bvec{Y} | \bvec{Z}}(\bvec{y} | \bvec{T}^{-1}(\bm{\theta})) p_{\bvec{Z}}(\bvec{T}^{-1}(\bm{\theta}))}{q_{\bvec{Y}}(\bvec{y})} \frac{p_{\bvec{\Theta}}(\bm{\theta})}{p_{\bvec{Z}}(\bvec{T}^{-1}(\bm{\theta}))} = \frac{q_{\bvec{Y} | \bvec{Z}}(\bvec{y} | \bvec{T}^{-1}(\bm{\theta})) p_{\bvec{\Theta}}(\bm{\theta})}{q_{\bvec{Y}}(\bvec{y})}.
					\end{equation}
					\subitem Note that the \gls{MKML} $q_{\bvec{Y}}(\bvec{y})$ is still marginalized over the simpler Gaussian distribution,
					\begin{equation}
						q_{\bvec{Y}}(\bvec{y}) = \int_{\mathcal{Z}} q_{\bvec{Y} | \bvec{Z}}(\bvec{y} | \bvec{z}) p_{\bvec{Z}}(\bvec{z}) d\bvec{z}.
					\end{equation}	
			\end{enumerate}
			
	In this way, we simplify the \gls{LFI} problem into another \gls{LFI} problem which involves a Gaussian prior such that \gls{KELFI} solutions are closed-form under Gaussian kernels. Once \gls{KELFI} solutions have been computed in the new parameter space $\mathcal{Z}$, the solutions can be easily transformed back into the original parameter space $\vartheta$ as above.
	
	This process is possible since the likelihood is intractable already. Hence, transformations $\bvec{T}$ of variables $\bvec{z}$ into simulator parameters $\bm{\theta}$ can be included as part of the simulator without changing the nature of the problem. 
	
	If simulation pairs $\{\bm{\theta}^{(j)}, \bvec{x}^{(j)}\}_{j = 1}^{m}$ in the original space are already provided, parameters $\bm{\theta}^{(j)}$ can be converted into Gaussian variables via $\bvec{z}^{(j)} = \bvec{T}^{-1}(\bm{\theta}^{(j)})$ for $j \in [m]$ so that the pairs $\{\bvec{z}^{(j)}, \bvec{x}^{(j)}\}_{j = 1}^{m}$ can be used to proceed.
	
	As an extension, instead of transforming the \gls{LFI} problem with a general continuous prior into one with a Gaussian prior, if the prior is fundamentally multi-modal, we can also transform it into one with a Gaussian mixture model as the prior. Since the prior density is a linear combination of Gaussians, all derivations remain closed-form from a linear combination of the results with each Gaussian component.
		
	Finally, it is important to recognize that while there is no loss of generality to the inference problem when performing this prior transform, the transformation do change the interpretation of the hyperparameters learned with the \gls{MKML}. Since the kernel $\ell$ is now placed in the $\mathcal{Z}$ space, the hyperparameters of $\ell$ cannot be interpreted directly for the original parameter space $\vartheta$ unless the transformation between $\mathcal{Z}$ and $\vartheta$ is simple enough to translate the interpretation. Nevertheless, hyperparameters can still be learned by optimizing the \gls{MKML}.

\end{document}